\documentclass[journal,twoside]{IEEEtran}
\usepackage{mathtools}
\usepackage{amsmath,epsfig}
\usepackage{amsthm}
\usepackage{amssymb}
\usepackage{bm}
\usepackage{algorithm}
\usepackage[noend]{algpseudocode}
\usepackage{bbm} 
\usepackage{color}

\ifCLASSINFOpdf
\else
\fi

\theoremstyle{my_style}
\newtheorem{theorem}{Theorem}
\newtheorem{lemma}{Lemma}
\newtheorem{proposition}{Proposition}
\newtheorem{remark}{Remark}

\newcommand{\cN}{{\mathcal{N}}}

\newcommand{\cH}{\mathcal{H}}

\newcommand{\cL}{\mathcal{L}}

\newcommand{\cD}{\mathcal{D}}

\newcommand{\bI}{\boldsymbol{I}}

\newcommand{\bA}{\boldsymbol{A}}
\newcommand{\bB}{\boldsymbol{B}}
\newcommand{\bX}{\boldsymbol{X}}
\newcommand{\bU}{\boldsymbol{U}}

\newcommand{\bbM}{\boldsymbol{M}}

\newcommand{\bG}{\boldsymbol{G}}

\newcommand{\bV}{\boldsymbol{V}}

\newcommand{\bx}{\boldsymbol{x}}

\newcommand{\by}{\boldsymbol{y}}
\newcommand{\bu}{\boldsymbol{u}}
\newcommand{\bv}{\boldsymbol{v}}
\newcommand{\bz}{\boldsymbol{z}}
\newcommand{\ba}{\boldsymbol{a}}
\newcommand{\bb}{\boldsymbol{b}}

\newcommand{\bc}{\boldsymbol{c}}

\newcommand{\bw}{\boldsymbol{w}}

\newcommand{\bg}{\boldsymbol{g}}
\newcommand{\bZero}{\boldsymbol{0}}

\newcommand{\bzero}{\boldsymbol{0}}

\newcommand{\bdelta}{\boldsymbol{\delta}}
\newcommand{\btheta}{\boldsymbol{\theta}}
\newcommand{\bpsi}{\boldsymbol{\psi}}
\newcommand{\bomega}{\boldsymbol{\omega}}

\newcommand{\beq}{\begin{equation}}
\newcommand{\eeq}{\end{equation}}
\newcommand{\beqn}{\begin{eqnarray}}
\newcommand{\eeqn}{\end{eqnarray}}
\newcommand{\beqns}{\begin{eqnarray*}}
\newcommand{\eeqns}{\end{eqnarray*}}

\newcommand{\R}{\mathbb{R}}

\newcommand{\N}{\mathbb{N}}

\newcommand{\pend}{\hfill$\square$\\}

\makeatletter
\newcommand{\bdiv}{\mathop{\operator@font div}}
\makeatother

\makeatletter
\newcommand{\diag}{\mathop{\operator@font diag}}
\makeatother

\makeatletter
\newcommand{\conv}{\mathop{\operator@font conv}}
\makeatother

\makeatletter
\newcommand{\minim}{\mathop{\operator@font minimize}}
\makeatother

\makeatletter
\newcommand{\maxim}{\mathop{\operator@font maximize}}
\makeatother

\makeatletter
\newcommand{\sign}{\mathop{\operator@font sign}}
\makeatother \makeatletter

\makeatletter
\newcommand{\proj}{\mathop{\operator@font proj}}
\makeatother \makeatletter

\makeatletter
\newcommand{\spa}{\mathop{\operator@font span}}
\makeatother \makeatletter

\makeatletter
\newcommand{\epi}{\mathop{\operator@font epi}}
\makeatother \makeatletter

\makeatletter
\newcommand{\dom}{\mathop{\operator@font dom}}
\makeatother \makeatletter

\makeatletter
\newcommand{\real}{\mathop{\operator@font Re}}
\makeatother \makeatletter

\makeatletter
\newcommand{\imag}{\mathop{\operator@font Im}}
\makeatother \makeatletter

\makeatletter
\newcommand{\sinc}{\mathop{\operator@font sinc}}
\makeatother \makeatletter

\makeatletter
\newcommand{\trace}{\mathop{\operator@font tr}}
\makeatother \makeatletter

\makeatletter
\newcommand{\vecbr}{\mathop{\operator@font vecb_r}}
\makeatother

\makeatletter
\newcommand{\veco}{\mathop{\operator@font vec}}
\makeatother

\newtheorem*{myassm*}{Assumption}
\newtheorem{myassm}{Assumption}

\begin{document}
%
\title{Online Distributed Learning Over Networks in RKH Spaces Using Random Fourier Features}
%
%
%

\author{Pantelis~Bouboulis,~\IEEEmembership{Member,~IEEE,}
        Symeon~Chouvardas,~\IEEEmembership{Member,~IEEE,}
        and~Sergios~Theodoridis,~\IEEEmembership{Fellow,~IEEE}
\thanks{P. Bouboulis and S. Theodoridis are with the Department
of Informatics and Telecommunications, University of Athens, Greece,
e-mails: panbouboulis@gmail.com, stheodor@di.uoa.gr.}
\thanks{S. Chouvardas is with the Mathematical and Algorithmic Sciences Lab
France Research Center, Huawei Technologies Co., Ltd.,
e-mail: symeon.chouvardas@huawei.com}
}

\maketitle

\begin{abstract}
We present a novel diffusion scheme for online kernel-based learning over networks. So far, a major drawback of any online learning algorithm, operating in a reproducing kernel Hilbert space (RKHS), is the need for updating a growing number of parameters as time iterations evolve. Besides complexity, this leads to an increased need of communication resources, in a distributed setting. In contrast, the proposed method approximates the solution as a fixed-size vector (of larger dimension than the input space) using Random Fourier Features. This paves the way to use standard linear combine-then-adapt techniques. To the best of our knowledge, this is the first time that a complete protocol for distributed online learning in RKHS is presented. Conditions for asymptotic convergence and boundness of the networkwise regret are also provided. The simulated tests illustrate the performance of the proposed scheme.
\end{abstract}


%
\IEEEpeerreviewmaketitle

\section{Introduction}
%
%
%
%

\IEEEPARstart{T}{HE} topic of distributed learning,
has grown rapidly over the last
years.
This is mainly due to the exponentially increasing volume of data,
that leads, in turn,
to increased requirements for memory
and computational resources.
Typical applications include sensor networks,
social networks, imaging, databases,
medical platforms, e.t.c., \cite{SlGiMa14}.
In most of those,
the data cannot be processed
on a single processing
unit (due to memory and/or computational
power constrains)
and  the
respective learning/inference
problem
has to be split into subproblems.
Hence, one has to resort to
distributed algorithms,  which operate on data that
are not available on a single location
but are instead spread out over
multiple locations, e.g.,   \cite{ChSaLi07,agrawal2011big,wu2014data}.

In this paper, we focus on the topic of
\textit{distributed online learning} and
in particular to non linear parameter estimation and classification tasks.
More specifically, we consider a decentralized network which comprises of
nodes, that observe data generated by a non linear model in a sequential fashion.
Each node communicates its own estimates of the unknown parameters to its neighbors and exploits simultaneously a)
the information that it receives and b) the observed datum, at each time instant, in order to update the associated with it estimates.
Furthermore, no assumptions are made regarding the presence of a central node, which could
perform all the necessary operations. Thus, the nodes act as independent learners
and perform the computations by themselves.
Finally, the task of interest is considered to be common across the nodes
and, thus, cooperation among each other is
meaningful and beneficial, \cite{Sa13, Theo_ML}.

The problem of linear  online estimation has been
 considered in several works.
 These include diffusion-based algorithms, e.g., \cite{ChSlTh12,Lopes,Cavalcante},
 ADMM-based schemes, e.g., \cite{ScMaGi09,MaScGi09}, as well as consencus-based ones, e.g., \cite{bertsekas1989parallel,dimakis2010gossip}. The multitask learning problem, in which
 there are more than one parameter vectors to
 be estimated, has also been treated, e.g.,
 \cite{ChCeSa14,PlBoBe15}.
The literature on online distributed classification
 is more limited; in \cite{forero2010consensus}, a batch distributed SVM algorithm
 is presented,  whereas in \cite{sayedclassification}, a diffusion based scheme
 suitable for classification is proposed. In the latter, the authors study the problem of distributed online learning  focusing on   strongly-convex risk functions, such as the logistic regression loss,
 which is suitable to tackle classification tasks. The nodes of the network cooperate via the diffusion rationale.
In contrast to the vast majority of works on the topic of distributed online learning,
which assume a linear relationship between input and output measurements,
in this paper we tackle the more general problem, i.e., the distributed online \textit{non--linear}
learning task. To be more specific, we assume that the data are generated by a model $y = f(\bx)$,
where $f$ is a non-linear function that lies in a \textit{Reproducing Kernel Hilbert Space} (RKHS).
These are inner-product function spaces, generated by a specific kernel function, that have
become popular models for non-linear tasks, since the introduction of the celebrated Support Vectors Machines
(SVM) \cite{Scholkopf_2002_2276, ShaweTaylor_2004_9255, Theodoridis_2008_9253, Theo_ML}.

Although there have been methods that attempt to generalize
linear online distributed strategies to the non-linear domain using RKHS, mainly in the context of the Kernel LMS
e.g., \cite{Distributed_KLMS1, Distributed_KLMS2, Distributed_KLMS3}, these have major drawbacks.
In \cite{Distributed_KLMS1} and \cite{Distributed_KLMS3}, the estimation of $f$, at each node, is given as an increasingly
growing sum of kernel functions centered at the observed data. Thus, a) each node has to transmit
the entire sum at each time instant to its neighbors and b) the node has to fuse together all sums received by its
neighbors to compute the new estimation. Hence, both the communications load of the entire network as well as the
computational burden at each node grow linearly with time. Clearly, this is impractical for real life applications.
In contrast, the method of \cite{Distributed_KLMS2} assumes that these growing sums are limited by a sparsification strategy;
how this can be achieved is left for the future. Moreover, the aforementioned methods offer no theoretical results regarding
the consensus of the network. In this work, we present a complete protocol for distributed online non-linear learning for both regression and classification tasks,
overcoming the aforementioned problems.
 Moreover, we present theoretical results regarding network-wise consensus and regret bounds. The proposed framework offers fixed-size communication and computational load
as time evolves. This is achieved through an efficient approximation of the growing sum using the random Fourier features rationale \cite{RahimiRecht}. To the best of our knowledge, this is the first time that such a method appears in the literature.

Section \ref{SEC:prelim} presents a brief background on kernel online methods and summarizes the main tools and notions used in this manuscript. The main contributions of the paper are presented in section \ref{SEC:Distributed}. The proposed method, the related theoretical results and extensive experiments can be found there. Section \ref{SEC:OKL} presents a special case of the proposed framework for the case of a single node. In this case, we demonstrate how the proposed scheme can be seen as a fixed-budget
alternative for online kernel based learning (solving the problem of the growing sum). Finally, section \ref{SEC:CONCL} offers some concluding remarks. In the rest of the paper, boldface symbols denote vectors, while capital letters are reserved for matrices. The symbol $\otimes$ denotes the Kronecker product of matrices and the symbol $\cdot^T$ the transpose of the respective matrix or vector. Finally, the symbol $\|\cdot\|$ refers to the respective $\ell_2$ matrix or vector norm.

\section{Preliminaries}\label{SEC:prelim}
\subsection{RKHS}
Reproducing Kernel Hilbert Spaces (RKHS) are inner product spaces of functions defined on $X$, whose respective point evaluation functional, i.e., $T_x:\cH\rightarrow X: T_x(f) = f(x)$, is linear and continuous for every $x\in X$. This is usually portrayed by the \textit{reproducing property} \cite{Scholkopf_2002_2276, Theo_ML, Cristianini_2000_11710}, which links inner products in $\cH$ with a specific (semi-)positive definite kernel function $\kappa$ defined on $X\times X$ (associated with the space $\cH$). As $\kappa(\cdot, x)$ lies in $\cH$ for all $x\in X$, the reproducing property declares that $\langle \kappa(\cdot, y), \kappa(\cdot, x)\rangle_{\cH} = \kappa(x,y)$, for all $x,y\in X$. Hence, linear tasks, defined on the high dimensional space, $\cH$, (whose dimensionality can also be infinite) can be equivalently viewed as non-linear ones on the, usually, much lower dimensional space, $X$, and vice versa. This is the essence of the so called \textit{kernel trick}: Any kernel-based learning method can be seen as a two step procedure, where firstly the original data are transformed from $X$ to $\cH$, via an implicit map, $\Phi(x) = \kappa(\cdot, x)$, and then linear algorithms are applied to the transformed data. There exist a plethora of different kernels to choose from in the respective literature. In this paper, we mostly focus on the popular Gaussian kernel, i.e., $\kappa(\bx,\by) = e^{\|\bx-\by\|^2/(2\sigma^2)}$, although any other shift invariant kernel can be adopted too.

Another important feature of RKHS is that any regularized ridge regression task, defined on $\cH$, has a unique solution, which can written in terms of a finite expansion of kernel functions centered at the training points. Specifically, given the set of training points $\{(x_n, y_n), n=1,\dots, N, x_n\in X, y_n \in \R\}$, the \textit{representer theorem} \cite{Wahba, Scholkopf_2002_2276}, states that the unique minimizer, $f_*\in\cH$, of $\sum_{n=1}^N l(f(x_n),y_n)  + \lambda\|f\|^2_{\cH}$, admits a representation of the form $f_* = \sum_{n=1}^N a_n \kappa(\cdot, x_n)$, where $l$ is any convex loss function that measures the error between the actual system's outputs, $y_n$, and the estimated ones, $f(x_n)$, and $\|\cdot\|_{\cH}$ is the norm induced by the inner product.

\subsection{Kernel Online Learning}
The aforementioned properties have rendered RKHS a popular tool for addressing non linear tasks both in batch and online settings. Besides the widely adopted application on SVMs, in recent years there has been an increased interest on non linear online tasks around the squared error loss function. Hence, there have been kernel-based implementations of LMS \cite{Liu_2008_10645, Bouboulis_2011_10643}, RLS \cite{Vae06, Engel_2004_11231}, APSM \cite{Slavakis_2008_9258, Slavakis_2011_11460} and other related methods \cite{Slavakis_2008_9257}, as well as online implementations of SVMs \cite{Shalev-Shwartz2011}, focusing on the primal formulation of the task.  Henceforth, we will consider online learning tasks based on the training sequences of the form $\cD = \{(\bx_n, y_n),\; n=1,2,\dots\}$, where $\bx_n\in\R^d$ and $y_n\in\R$. The goal of the assumed learning tasks is to learn a non-linear input-output dependence, $y=f(\bx)$, $f\in\cH$, so that to minimize a preselected cost. Note that these types of tasks include both classification (where $y_n = \pm 1$) and regression problems (where $y_n\in\R$). Moreover, in the online setting, the  data are assumed to arrive sequentially.

As a typical example of these tasks, we consider the KLMS, which is one of the simplest and most representative methods of this kind. Its goal is to learn $f$, so that to minimize the MSE, i.e., $\cL(f) = E[(y - f(\bx))^2]$. Computing the gradient of $\cL$ and estimating it via the current set of observations (in line with the stochastic approximation rationale, e.g., \cite{Theo_ML}), the estimate at the next iteration, employing the  gradient descent method, becomes
$f_n = f_{n-1} + \mu \epsilon_n \kappa(\bx_n,\cdot)$,
where $\epsilon_n=y_n - f_{n-1}(\bx_n)$ and $\mu$ is the step-size (see, e.g., \cite{Theo_ML, EREF2, Liu_2010_10644} for more). Assuming that the initial estimate is zero, the solution after $n-1$ steps turns out to be
\begin{align}\label{KLMS_solution}
f_{n-1} = \sum_{i=1}^{n-1} \alpha_i \kappa(\cdot,\bx_i),
\end{align}
where $\alpha_i = \mu \epsilon_i$. Observe that this is in line with the representer theorem. Similarly, the system's output can be estimated as $f_{n-1}(\bx_n) = \sum_{i=1}^{n-1} \alpha_i \kappa(\bx_n,\bx_i)$. Clearly, this linear expansion grows indefinitely as $n$ increases; hence the original form of KLMS is impractical. Typically, a sparsification strategy is adopted to bound the size of the expansion \cite{RichBerm, RichBerm2, QKLMS}. In these methods, a specific criterion is employed to decide whether a particular point, $\bx_n$, is to be included to the expansion, or (if that point is discarded) how its respective output $y_n$ can be exploited to update the remaining weights of the expansion. There are also methods that can remove specific points from the expansion, if their information becomes obsolete, in order to increase the tracking ability of the algorithm \cite{SOQKLMS}.

\subsection{Approximating the Kernel with random Fourier Features}\label{SEC:RFF}
Usually, kernel-based learning methods involve a large number of kernel evaluations between training samples. In the batch mode of operation, for example, this means that a large kernel matrix has to be computed, increasing the computational cost of the method significantly.
Hence, to alleviate the computational burden, one common approach is to use some sort of approximation of the kernel evaluation. The most popular techniques of this category are the Nystr\"{o}m method \cite{Nystrom_kernel, Nystrom2} and the random Fourier features approach \cite{RahimiRecht, RahimiRecht2}; the latter fits naturally to the online setting. Instead of relying on the implicit lifting, $\Phi$, provided by the kernel trick, Rahimi and Recht in \cite{RahimiRecht} proposed to map the input data to a finite-dimensional Euclidean space (with dimension lower than $\cH$ but larger than the input space) using a randomized feature map $\bz_\Omega: \R^d \rightarrow \R^D$, so that the kernel evaluations can be approximated as $\kappa(\bx_n, \bx_m) \approx \bz_\Omega(\bx_n)^T \bz_\Omega(\bx_m)$.
The following theorem plays a key role in this procedure.

\begin{theorem}\label{THE:rff}
Consider a shift-invariant positive definite kernel $\kappa(\bx-\by)$ defined on $\R^d$ and its Fourier transform $p(\bomega) = \frac{1}{(2\pi)^d}\int_{\R^d} \kappa(\bdelta) e^{-i\bomega^T\bdelta} d\bdelta$, which (according to Bochner's theorem) it can be regarded as a \textbf{probability density} function. Then, defining $z_{\bomega, b}(\bx) = \sqrt{2}\cos(\bomega^T\bx + b)$, it turns out that
\begin{align}
\kappa(\bx-\by) = E_{\bomega, b}[z_{\bomega, b}(\bx) z_{\bomega, b}(\by)],
\end{align}
where $\bomega$ is drawn from $p$ and $b$ from the uniform distribution on $[0,2\pi]$.
\end{theorem}

Following Theorem \ref{THE:rff}, we choose to approximate $\kappa(\bx_n-\bx_m)$ using $D$ random Fourier features, $\bomega_1, \bomega_2, \dots, \bomega_D$, (drawn from $p$) and $D$ random numbers, $b_1, b_2, \dots, b_D$ (drawn uniformly from $[0, 2\pi]$) that define a sample average:
\begin{align}\label{EQ:approx1}
\kappa(\bx_n-\bx_m) \approx \frac{1}{D}\sum_{i=1}^D z_{\bomega_i, b_i}(\bx_m) z_{\bomega_i, b_i}(\bx_n).
\end{align}
Evidently, the larger $D$ is, the better this approximation becomes (up to a certain point). Details on the quality of this approximation can be found in \cite{RahimiRecht, RahimiRecht2, Sutherland, RFF_nystrom}.  We note that for the Gaussian kernel, which is employed throughout the paper, the respective Fourier transform is
\begin{align}\label{EQ:fourier_of_gaussian}
p(\bomega) = \left(\sigma/\sqrt{2\pi}\right)^D e^{-\frac{\sigma^2\|\bomega\|^2}{2}},
\end{align}
which is actually the multivariate Gaussian distribution with mean $\bzero_D$ and covariance matrix $\frac{1}{\sigma^2}\bI_D$.

We will demonstrate how this method can be applied using the KLMS paradigm. To this end, we define the map $\bz_\Omega:\R^d \rightarrow \R^D$ as follows:
\begin{align}\label{EQ:Z_map}
\bz_\Omega(\bu) = \sqrt{\frac{2}{D}}\left(\begin{matrix} \cos(\bomega_1^T\bu +b_1) \cr \vdots \cr
\cos(\bomega_D^T\bu +b_D) \end{matrix}\right),
\end{align}
where $\Omega$ is the $(d+1)\times D$ matrix defining the random Fourier features of the respective kernel, i.e.,
\begin{align}\label{EQ:omega}
\Omega = \left(
\begin{matrix}
\bomega_1 & \bomega_2 & ...  & \bomega_D\cr
b_1  & b_2 & ... & b_D
\end{matrix}
\right),
\end{align}
provided that $\bomega$'s and $b$'s are drawn as described in theorem \ref{THE:rff}. Employing this notation, it is straightforward to see that \eqref{EQ:approx1} can be recast as
$\kappa(\bx_n-\bx_m) \approx \bz_\Omega(\bx_m)^T \bz_\Omega(\bx_n)$.
Hence, the output associated with observation $\bx_n$ can be approximated as
\begin{align}\label{EQ:general_output_approx}
f_{n-1}(\bx_n) &\approx 
 \left(\sum_{i=1}^{n-1} \alpha_i \bz_\Omega(\bx_i)\right)^T \bz_\Omega(\bx_n).
\end{align}
It is a matter of elementary algebra to see that \eqref{EQ:general_output_approx} can be equivalently derived by approximating the system's output as $f(\bx) \approx \btheta^T \bz_\Omega(\bx)$, initializing $\btheta_0$ to $\bZero_D$ and iteratively applying the following gradient descent type update: $\btheta_n = \btheta_{n-1} + \mu e_n\bz_\Omega(\bx_n)$.

Clearly, the procedure described here, for the case of the KLMS, can be applied to any other gradient-type kernel based method. It has the advantage of modeling the solution as a fixed size vector, instead of a growing sum, a property that is quite helpful in distributed environments, as it will be discussed in section \ref{SEC:Distributed}.

\section{Distributed kernel-based Learning}\label{SEC:Distributed}
In this section, we discuss the problem of online learning in RKHS over distributed networks. Specifically, we consider $K$ connected nodes, labeled $k\in\cN=\{1,2,\dots K\}$, which operate in cooperation with their neighbors to solve a specific task. Let $\cN_k\subseteq\cN$ denote the neighbors of node $k$. The network topology is represented as an undirected \textit{connected} graph, consisting of $K$ vertices (representing the nodes) and a set of edges connecting the nodes to each other (i.e., each node is connected to its neighbors). We assign a nonnegative weight $a_{k,l}$ to the edge connecting node $k$ to $l$. This weight is used by $k$ to scale the data transmitted from $l$ and vice versa. This can be interpreted as a measure of the confidence level that node $k$ assigns to its interaction with node $l$. We collect all coefficients into a $K\times K$ symmetric matrix $A = (a_{k,l})$, such that the entries of the $k$-th row of $A$ contain the coefficients used by node $k$ to scale the data arriving from its neighbors. We make the additional assumption that $A$ is doubly stochastic, so that the weights of all incoming and outgoing ``transmissions" sum to 1. A common choice, among others, for choosing these coefficients, is the \textit{Metropolis rule}, in which the weights equal to:
\begin{align*}
 a_{k,l}= \begin{cases} \frac{1}{\max\lbrace\vert\cN_k\vert,\vert\cN_l\vert\rbrace}, \ &\mathrm{if} \ l\in\cN_k, \ \mathrm{and} \ l\neq k\\ 1- \sum_{i\in\cN_k \setminus {k}} a_{k,i}, \ &\mathrm{if} \ l=k\\ 0, &\ \mathrm{otherwise}. \end{cases}
\end{align*}
Finally, we assume that each node, $k$, receives streaming data $\{(\bx_{k,n}, y_{k,n}), n=1,2,\dots\}$, that are generated from an input-output relationship of the form $y_{k,n} = f(\bx_{k,n}) + \eta_{k,n}$, where $\bx_{k,n}\in\R^d$, $y_{k,n}$ belongs to $\R$ and $\eta_{k,n}$ represents the respective noise, for the regression task. The goal is to obtain an estimate of $f$. For classification, $y_{n,k}=\phi(f(\bm{x}_{k,n}))$, where, $\phi$ is a thresholding function; here we assume that $y_{n,k}\in \{-1,1\}$. Once more, the goal is to optimally estimate the classifier function $f$.

Each one of the nodes aims to estimate $f\in\cH$ by minimizing a specific convex cost function, $\cL(\bx,y,f)$, using a (sub)gradient descent approach.
We employ a simple \textit{Combine-Then-Adapt} (CTA) rationale, where at each time instant, $n$, each node, $k$, a) receives the current estimates, $f_{l,n-1}$, from all neighbors (i.e., from all nodes $l\in\cN_k$), b) combines them to a single solution, $\psi_{k,n-1} = \sum_{l\in\cN_k}a_{k,l} f_{l, n-1}$ and c) apply a step update procedure:
\begin{align*}
f_n=\psi_{k,n-1} -\mu_n\nabla_f\cL(\bx_n,y_n,\psi_{k,n-1}).
\end{align*}
The implementation of such an approach in the context of RKHS presents significant challenges. Keep in mind that, the estimation of the solution at each node is not a simple vector, but instead it is a function, which is expressed as a growing sum of kernel evaluations centered at the points observed by the specific node, as in \eqref{KLMS_solution}. Hence, the implementation of a straightforward CTA strategy would require from each node to transmit its entire growing sum (i.e., the coefficients $a_i$ as well as the respective centers $\bx_i$) to all neighbors. This would significantly increase both the communication load among the nodes, as well as the computational cost at each node, since the size of each one of the expansions would become increasingly larger as time evolves (as for every time instant, they gather the centers transmitted by all neighbors).  This is the rationale adopted in \cite{Distributed_KLMS1, Distributed_KLMS2, Distributed_KLMS3} for the case of KLMS.
Clearly, this is far from a practical approach. Alternatively, one could devise an efficient method to sparsify the solution at each node and then merge the sums transmitted by its neighbors. This would require (for example) to search all the dictionaries, transmitted by the neighboring nodes, for similar centers and treat them as a single one, or adopting a single pre-arranged dictionary (i.e., a specific set of centers) for all nodes and then fuse each observed point with the best-suited center. However, no such strategy has appeared in the respective literature, perhaps due to its increased complexity and lack of a theoretical elegance.

In this paper, inspired by the random Fourier features approximation technique, we approximate the desired input-output relationship as $y = \btheta^T\bz_\Omega(\bx)$ and propose a two step procedure: a) we map each observed point $(\bx_{k,n}, y_{k,n})$ to $(\bz_\Omega(\bx_{k,n}), y_{k,n})$ and then b) we adopt a simple linear CTA diffusion strategy on the transformed points. Note that in the proposed scheme, each node aims to estimate a vector $\btheta\in\R^D$ by minimizing a specific (convex) cost function, $\cL(\bx,y,\btheta)$. Here, we imply that the model can be closely approximated by $y_{k,n} \approx \btheta^T\bz_\Omega(\bx_{k,n}) + \eta_{k,n}$, for regression, and $y_{k,n}\sim \phi(\bm{\theta}^T\bm{z}_\Omega(\bm{x}_{k,n}))$ for classification, for all $k,n$, for some $\btheta$. We emphasize that $\cL$ need not be differentiable.  Hence, a large family of loss functions can be adopted. For example:
\begin{itemize}
\item Squared error loss: $\cL(\bx,y,\btheta) = (y - \btheta^T\bx)^2$.
\item Hinge loss:  $\cL(\bx,y,\btheta) = \max(0, 1-y\btheta^T\bx)$.
\end{itemize}
We end up with the following generic update rule:
\begin{align}
\bpsi_{k,n} &= \sum_{l\in\cN_k} a_{k,l} \btheta_{l,n-1}, \label{EQ:CTA1}\\
\btheta_{k,n} &= \bpsi_{k,n}   - \mu_{k,n} \nabla_{\btheta} \cL(\bz_\Omega(\bx_{k,n}), y_{k,n},\bpsi_{k,n}),\label{EQ:CTA2}
\end{align}
where $\nabla_{\btheta} \cL(\bz_\Omega(\bx_{k,n}), y_{k,n},\bpsi_{k,n})$ is the gradient, or any subgradient of $\cL(\bx,y,\btheta)$ (with respect to $\btheta$), if the loss function is not differentiable. Algorithm \ref{Alg:RFF-DOKL} summarizes the aforementioned procedure.
The advantage of the proposed scheme is that each node transmits a single vector (i.e., its current estimate, $\btheta_{k,n}$) to its neighbors, while the merging of the solutions requires only a straightforward summation.

\begin{algorithm}[t]
\caption{Random Fourier Features Distributed Online Kernel-based Learning (RFF-DOKL).}\label{Alg:RFF-DOKL}
\begin{algorithmic}
\State $D = \{(\bx_{k,n}, y_{k,n}), k=1,2\dots,K,\;n=1,2,\dots\}$ \Comment{Input}
\State Select a specific shift invariant (semi)positive definite kernel, a specific loss function $\cL$ and a sequence of possible variable learning rates $\mu_n$. Each node generates the same matrix $\Omega$ as in \eqref{EQ:omega}.
\State $\btheta_{k,0} \gets \bZero_D$, for all $k$. \Comment{Initialization}
\For{$n=1,2,3, ...$}
\For{each node $k$}
\State $\bpsi_{k,n} = \sum_{l\in\cN_k} a_{k,l} \btheta_{l,n-1}$.
\State $\btheta_{k,n} = \bpsi_{k,n}  - \mu_{k,n} \nabla_{\btheta} \cL(\bz_\Omega(\bx_{k,n}), y_{k,n},\bpsi_{k,n})$.
\EndFor
\EndFor
\end{algorithmic}
\end{algorithm}

\subsection{Consensus and regret bound}\label{SEC:CONSENSUS}
In the sequel, we will show that, under certain assumptions, the proposed scheme achieves asymptotic consensus and that the corresponding regret bound grows sublinearly with the time. It can  readily be seen that \eqref{EQ:CTA1}-\eqref{EQ:CTA2}
can be written more compactly (for the whole network) as follows:
\begin{equation}
\underline{\btheta}_n =\bA\underline{\btheta}_{n-1}-\bbM_n \bG_n ,\label{EQ:CTA}
\end{equation}
where $\underline{\btheta}_n := (\btheta_{1,n}^T,\ldots, \btheta_{K,n}^T)^T \in \R^{KD}$, $\bbM_n := \mathrm{diag}\{\mu_{1,n},\ldots,\mu_{K,n} \}\otimes I_D$,
$\bG_n: = [ (\bu_{1,n}^T,\ldots, \bu_{K,n}^T]^T\in\R^{KD}$, where $\bu_{k,n} = \nabla\cL(\bz_\Omega(\bx_{k,n}),y_{k,n},\bpsi_{k,n})$, and $\bA := A\otimes I_D$. The necessary assumptions are the following:

\begin{myassm}
The step size is time decaying and is bounded by the inverse square root of time, i.e.,
 $\mu_{k,n}=\mu_n \leq \mu n^{-1/2}$.
\end{myassm}
\begin{myassm}
The norm of the transformed input is bounded, i.e.,
$\exists U_1$ such that
 $\|\bz_{\Omega}(\bx_{k,n})\|\leq U_1, \ \forall k\in\cN, \forall n\in\N$.
 Furthermore, $y_{k,n}$  is bounded, i.e., $|y_{k,n}|\leq V\ \forall k\in\cN, \forall n\in\mathbb{N}$ for some $V>0$.
\end{myassm}
\begin{myassm}
The estimates are bounded, i.e., $\exists U_2$
 s.t. $\| \btheta_{k,n}\|\leq U_2,\ \forall k\in\cN, \forall n\in\N$.
\end{myassm}
\begin{myassm}
The matrix comprising the combination weights, i.e., $A$, is doubly stochastic (if the weights are chosen with respect to the Metropolis rule, this condition is met).
\end{myassm}

\noindent Note that assumptions 2 and 3 are valid for most of the popular cost functions. As an example, we can study the squared error loss, i.e., $\cL(\bx,y,\btheta) = 1/2(y - \btheta^T\bx)^2$, where:
\begin{align*}
\| \nabla \cL(\bz_\Omega(\bx),y,\btheta)\| &\leq | y | \|\bz_\Omega(\bx)\Vert + \|\btheta\| \|\bz_\Omega(\bx)\|^2 \\
&\leq  V U_1 +  U_1^2U_2.
\end{align*}
Following similar arguments, we can also prove that many other popular cost functions (e.g., the hinge loss, the logistic loss, e.t.c.) have bounded gradients too.

\begin{proposition}[Asymptotic Consensus]\label{PRO:consensus}
All nodes converge to the same solution.
\end{proposition}

\begin{proof}
Consider a $KD\times KD$ \textit{consensus} matrix $\bA$ as in \eqref{EQ:CTA}. As $\bA$ is doubly stochastic, we have the following  \cite{Cavalcante}:
\begin{itemize}
\item $\|\bA \| = 1$.
\item Any consensus matrix $\bA$ can be decomposed as
\begin{align}
\bA= \bX + \bB \bB^T,\label{EQ:decomp}
\end{align}
where $\bB=[\bb_{1},\ldots,\bb_{D}]$ is an $KD\times D$ matrix,
and $\bb_{k}=1/\sqrt{K}(\bm{1}\otimes\bm{e}_{k})$, where
$\bm{e}_{k}$, $k=1,\dots, D$ represent the standard basis of $\R^D$ and $\bX$ is a $KD\times KD$ matrix for which it holds that $\left\| \bX \right\| <1$.
\item  $\bA \underline{\breve{\btheta}}=\underline{\breve{\btheta}}$, for all $\underline{\breve{\btheta}}\in \mathcal{O}:= \lbrace\underline{\btheta} \in \R^{KD}: \ \underline{{\btheta}}= [ \btheta^T,\ldots,\btheta^T ]^T, \  \btheta \in \R^D \rbrace$. The subspace $\mathcal{O}$ is the so called consensus subspace of dimension $D$, and $\bm{b}_k,\ k=1,\ldots,D$, constitute a basis for this space. Hence, the orthogonal projection of a vector,  $\underline{\btheta}$, onto this linear subspace is given by $P_{\mathcal{O}}(\underline{\btheta}):=\bB\bB^T\underline{\btheta}$, for all $\underline{\btheta}\in\R^{KD}$.
\end{itemize}

In \cite{Cavalcante}, it has been proved that, the algorithmic scheme achieves asymptotic consensus, i.e., $\|\btheta_{k,n}-\btheta_{l,n}\|\rightarrow 0$, as $n\rightarrow \infty,$ for all  $k,l\in\cN$, if and only if
$\lim_{n\rightarrow\infty}\|\bm{\underline{\btheta}}_n-P_{\mathcal{O}}(\bm{\underline{\btheta}}_n)\|={0}$.
We can easily check that the quantity
\begin{align}
\underline{\bm{r}}_n:=  \bm{\underline{\theta}}_{n+1}-\bm{A}\bm{\underline{\theta}}_n = -\bbM_{n+1}\bG_{n+1}.
\label{EQ:cons21}
\end{align}
approaches $0$, as $n\rightarrow\infty$, since $\lim_{n\to \infty}\bm{M}_n = \bm{O}_{KD}$ (assumption 1) and the matrix $\bm{G}_n$  is bounded for all $n$. Rearranging the terms of \eqref{EQ:cons21} and iterating over $n$, we have:
\begin{align*}
\bm{\underline{\theta}}_{n+1} &=\bA\bm{\underline{\theta}}_n+\bm{\underline{r}}_n
= \bA\bA \bm{\underline{\theta}}_{n-1} + \bA\bm{\underline{r}}_{n-1}+\bm{\underline{r}}_n=\ldots \\
&= \bA^{n+1}\bm{\underline{\theta}}_0 +   \sum_{j=0}^{n}  \bA^{n-j}\bm{\underline{r}}_{j}.
\end{align*}
If we left-multiply the previous equation by $(\bI_{KD}-\bB\bB^{T})$ and  follow similar steps as in \cite[Lemma 2]{Cavalcante}, it can be verified  that
$\underset{n \rightarrow \infty}{\lim} \|\left(\bm{I}_{Km}-\bB\bB^{T}\right)\bm{\underline{\theta}}_{n+1}\|=0$,
which completes our proof.
\end{proof}

\begin{proposition}\label{PRO:regret}
Under   assumptions 1-4 (and a cost function with bounded gradients) the networkwise regret is bounded by
\begin{align*}
\sum_{i=1}^N\sum_{k\in\cN}(\cL(\bx_{k,i},y_{k,i},\bpsi_{k,i})-\cL(\bx_{k,i},y_{k,i},\bm g)) \leq \gamma\sqrt{N} +\delta,
\end{align*}
for all $\bm{g}\in\mathcal{B}_{[\bZero_D,U_2]}$, where   $\gamma, \ \delta$ are positive constants and $\mathcal{B}_{[\bZero_D,U_2]}$ is the closed
ball with center $\bZero_D$ and radius $U_2$.
\end{proposition}
\begin{proof} See appendix \ref{SEC:proof_regret}. \end{proof}

\begin{remark}\label{REM:properties1}
It is worth pointing out that the theoretical properties, which were stated before, are complementary. In particular, the consensus property (Proposition \ref{PRO:consensus}) indicates that the nodes converge to the \textbf{same} solution and the sublinearity of the regret implies that on  average the algorithm performs as well as the best fixed strategy. In fact, without the regret related proof we cannot characterize  the solution in which the nodes converge.
\end{remark}

\subsection{Diffusion SVM (Pegasos) Algorithm}\label{SEC:DPeg}
The case of the regularized hinge loss function, i.e., $\cL(\bx, y, \btheta) = \frac{\lambda}{2}\|\btheta\|^2 + \max\{0, 1-y\btheta^T\bz_\Omega(\bx)\}$, for a specific value of the regularization parameter $\lambda>0$, generates the \textit{Distributed Pegasos} (see \cite{Shalev-Shwartz2011}). Note that the Pegasos solves the SVM task in the primal domain. In this case, the gradient becomes $\nabla_{\btheta}\cL(\bx,y,\btheta) = \lambda\btheta - \bm{I}_{+}(1-y\btheta^T\bz_\Omega(\bx))y\bz_\Omega(\bx)$, where $\bm{I}_{+}$ is the indicator function of $(0,+\infty)$, which takes a value of 1, if its argument belongs in $(0,+\infty)$, and zero otherwise. Hence the step-update equation of algorithm \ref{Alg:RFF-DOKL} becomes:
\begin{align}
\begin{matrix*}[l]
\btheta_{k,n} =& (1 - \frac{1}{n})\bpsi_{k,n-1}\\
&+ \bm{I}_{+}(1-y_n\bpsi_{k,n-1}^T\bz_\Omega(\bx_{k,n}))\frac{y_{k,n}}{\lambda n}\bz_\Omega(\bx_{k,n}),
\end{matrix*}
\end{align}
where, following \cite{Shalev-Shwartz2011}, we have used a decreasing step size, $\mu_n = \frac{1}{\lambda n}$. This scheme satisfies the required assumptions, hence consensus is guaranteed.

We have tested the performance of Distributed-Pegasos versus the non-cooperative Pegasus on
four datasets downloaded from Leon Bottou's LASVM web page \cite{Bottou_webpage}. The chosen datasets are: a) the Adult dataset, b) the Banana dataset (where we have used the first 4000 points as training data and the remaining 1300 as testing data), c) the Waveform dataset (where we have used the first 4000 points as training data and the remaining 1000 as testing data) and d) the MNIST dataset (for the task of classifying the digit 8 versus the rest). The sizes of the datasets are given in Table \ref{TAB:datasets}. In all experiments, we generate random graphs (using MIT's random\_graph routine, see \cite{MIT_Graph}) and compare the proposed diffusion method versus a noncooperative strategy (where each node works independent of the rest). For each realization of the experiments, a different random connected graph with $M=5$ or $M=20$ nodes was generated, with probability of attachment per node equal to 0.2 (i.e, there is a 20\% probability that a specific node $k$ is connected to any other node $l$). The adjacency matrix, $A$, of each graph was generated using the Metropolis rule. For the non-cooperative strategies, we used a graph that connects each node to itself, i.e., $A = I_{5}$ or $A = I_{20}$ respectively. The latter, implies that no information is exchanged between the nodes, thus each node is working alone.
Moreover, for each realization, the corresponding dataset was randomly split into $M$ subsets of equal size (one for every node).

We note that the value of $D$ affects significantly the quality of the approximation via the Fourier features rationale and thus it also affects the performance of the experiments. The value of $D$ must be large enough so that the approximation is good, but not too large so that to the communicational and computational load become affordable. In practice, we can find a value for $D$ so that any further increase results to almost negligible performance variation (see also section \ref{SEC:OKL}). All other parameters were optimized (after trials) to give the lowest number of test errors. Their values are reported on Table \ref{TAB:Peg_param}. The algorithms were implemented in MatLab and the experiments were performed on a i7-3770 machine running at $3.4$GHz with 32 Mb of RAM.
Tables \ref{TAB:DPEG1} and \ref{TAB:DPEG2} report the mean test errors obtained by both procedures. For $M=5$, the mean algebraic complexity of the generated graphs lies between $0.61$ and $0.76$ (different for each experiment), while the corresponding mean algebraic degree lies around 1.8. For $M=20$, the mean algebraic complexity of the generated graphs lies around $0.70$, while the corresponding mean algebraic degree lies around 3.9. The number inside the parentheses indicates the times of data reuse (i.e., running the algorithm again over the same data, albeit with a continuously decreasing step-size $\mu_n$), which has been suggested that improves the classification accuracy of Pegasos (see \cite{Shalev-Shwartz2011}). For example, the number 2 indicates that the algorithm runs over a dataset of double size, that contains the same data pairs twice.
For the three first datasets (Adult, Banana, Waveform) we have run 100 realizations of the experiment, while for the fourth (MNIST) we have run only 10 (to save time). Besides the ADULT dataset, all other simulations show that the distributed implementation significantly outperforms the non-cooperative one. For that particular dataset, we observe that for a single run the non-cooperative strategy behaves better (for $M=20$), but as data reuse increases the distributed implementation reaches lower error floors.

\begin{table}[t]
\scriptsize
\caption{Dataset Information.}\label{TAB:datasets}
\centering
\begin{tabular}{|c|c|c|c|c|}
\hline
Method   &   Adult &  Banana & Waveform  & MNIST\\\hline
Training size & 32562 & 4000 & 4000 & 60000 \\\hline
Testing size & 16282 & 1300 & 1000 & 10000 \\\hline
dimensions  & 123 & 2 & 21 & 784\\\hline
\end{tabular}
\end{table}

\begin{table}[t]
\scriptsize
\caption{Comparing the performances of the Distributed Pegasos versus the non-cooperative Pegasos for graphs with $M=5$ nodes.}\label{TAB:DPEG1}
\centering
\begin{tabular}{|c|c|c|c|c|}
\hline
Method   &   Adult  &  Banana & Waveform  & MNIST\\\hline
Distributed-Pegasos (1)      & 19\%    & 11.80\% & 11.82\% & 0.79\% \\\hline
Distributed-Pegasos (2)      & 17.43\% & 10.84\% & 10.49\% & 0.68\% \\\hline
Distributed-Pegasos (5)      & 15.87\% & 10.34\% & 9.56\%  & 0.59\% \\\hline
Non-cooperative-Pegasos (1)  & 19.11\% & 14.52\% & 13.75\% & 1.42\% \\\hline
Non-cooperative-Pegasos (2)  & 18.31\% & 12.52\% & 12.59\% & 1.19\% \\\hline
Non-cooperative-Pegasos (5)  & 17.29\% & 11.32\% & 11.86\% & 1.01\% \\\hline
\end{tabular}
\end{table}

\begin{table}[t]
\scriptsize
\caption{Comparing the performances of the Distributed Pegasos versus the non-cooperative Pegasos for graphs with $M=20$ nodes.}\label{TAB:DPEG2}
\centering
\begin{tabular}{|c|c|c|c|c|}
\hline
Method   &   Adult  &  Banana & Waveform  & MNIST\\\hline
Distributed-Pegasos (1)      & 24.04\% & 16.38\% & 16.26\% & 1.03\% \\\hline
Distributed-Pegasos (2)      & 22.34\% & 13.23\% & 13.93\% & 0.77\% \\\hline
Distributed-Pegasos (5)      & 18.94\% & 10.83\% & 11.20\% & 0.57\% \\\hline
Non-cooperative-Pegasos (1)  & 20.81\% & 21.74\% & 18.40\% & 2.93\% \\\hline
Non-cooperative-Pegasos (2)  & 20.52\% & 18.64\% & 16.54\% & 2.19\% \\\hline
Non-cooperative-Pegasos (5)  & 19.88\% & 15.96\% & 14.86\% & 1.87\% \\\hline
\end{tabular}
\end{table}

\begin{table}[t]
\scriptsize
\caption{Parameters for each method.}\label{TAB:Peg_param}
\centering
\begin{tabular}{|c|c|c|c|c|}
\hline
Method   &   Adult &  Banana & Waveform  & MNIST\\\hline
Kernel-Pegasos & $\begin{matrix} \sigma=\sqrt{10}\cr \lambda =  0.0000307\end{matrix}$ & $\begin{matrix} \sigma=0.7\cr \lambda = \frac{1}{316}\end{matrix}$ & $\begin{matrix} \sigma=\sqrt{10}\cr \lambda = 0.001\end{matrix}$ & $\begin{matrix} \sigma=4\cr \lambda = 10^{-7}\end{matrix}$  \\\hline
RFF-Pegasos &  $\begin{matrix} \sigma=\sqrt{10}\cr \lambda = 0.0000307 \cr D=2000\end{matrix}$ & $\begin{matrix} \sigma=0.7\cr \lambda = \frac{1}{316} \cr D=200\end{matrix}$ & $\begin{matrix} \sigma=\sqrt{10}\cr \lambda = 0.001\cr D=2000\end{matrix}$ & $\begin{matrix} \sigma=4\cr \lambda = 10^{-7}\cr D=100000\end{matrix}$ \\\hline
\end{tabular}
\end{table}

\subsection{Diffusion KLMS}\label{SEC:DKLMS}
Adopting the squared error in place of $\cL$, i.e., $\cL(\bx,y,\btheta) = (y - \btheta^T\bz_\Omega(\bx))^2$, and estimating the gradient by its current measurement, we take the Random Fourier Features Diffusion KLMS (RFF-DKLMS) and the step update becomes:
\begin{align}
\btheta_{k,n} = \bpsi_{k,n-1} + \mu \varepsilon_{k,n}\bz_\Omega(\bx_{k,n}), \label{EQ:DKLMS_update}
\end{align}
where $\varepsilon_{k,n} = y_n - \bpsi_{k,n-1}^T\bz_\Omega(\bx_{k,n})$. Although proposition \ref{PRO:consensus} cannot be applied here (as it requires a decreasing step-size), we can derive sufficient conditions for consensus following the results of the standard Diffusion LMS \cite{Lopes}. Henceforth, we will assume that the data pairs are generated  by
\begin{align}
y_{k,n} = \sum_{m=1}^M a_m\kappa(\bc_m, \bx_{k,n}) + \eta_{k,n},\label{EQ:graph_input_model}
\end{align}
where $\bc_1,\dots,\bc_M$ are fixed centers, $\bx_{k,n}$ are zero-mean i.i.d, samples drawn from the Gaussian distribution with covariance matrix $\sigma_x^2\bI_d$ and $\eta_{k,n}$ are i.i.d. noise samples drawn from $\cN(0, \sigma_\eta^2)$. Following the RFF approximation rationale (for shift invariant kernels), we can write that
\begin{align*}
y_{k,n} &= \sum_{m=1}^M a_m E_{\bomega,\bm{b}}[z_{\bomega,\bm{b}}(\bc_m)z_{\bomega,\bm{b}}(\bx_{k,n})] + \eta_{k,n} \\
&=\ba^T Z_\Omega^T\bz_{\Omega}(\bx_{k,n}) + \epsilon_{k,n} + \eta_{k,n},\\
&=\btheta_o^T \bz_{\Omega}(\bx_{k,n}) + \epsilon_{k,n} + \eta_{k,n},
\end{align*}
where $Z_\Omega = \left(\bz_{\Omega}(\bc_1), \dots, \bz_{\Omega}(\bc_M)\right)$, $\ba = (a_1, \dots, a_M)^T$, $\btheta_o=Z_\Omega\ba$ and $\epsilon_{k,n}$  is the approximation error between the noise-free component of $y_{k,n}$ (evaluated only by the linear kernel expansion of \eqref{EQ:graph_input_model}) and the approximation of this component using random Fourier features, i.e., $\epsilon_{k,n} = \sum_{m=1}^M a_m \kappa(\bc_m,\bx_{k,n}) - \btheta_o^T\bz_{\Omega}(\bx_{k,n})$.
For the whole network we have the following
\begin{align}
\underline{\by}_n = \bV^T_n\underline{\btheta}_o + \underline{\bm{\epsilon}}_n + \underline{\bm{\eta}}_n,\label{EQ:DKLMS_model}
\end{align}
where
\begin{itemize}
\item $\underline{\by}_n := (y_{1,n}, y_{2,n}, \dots, y_{K,n})^T$,
\item $\bV_n := \diag(\bz_\Omega(\bx_{1,n}), \bz_\Omega(\bx_{2,n}),\dots, \bz_\Omega(\bx_{K,n}))$, is a $DK\times K$ matrix,
\item $\underline{\btheta}_o = \left(\btheta_o^T, \btheta_o^T, \dots, \btheta_o^T\right)^T\in\R^{DK}$,
\item $\underline{\bm{\epsilon}}_n= \left(\epsilon_{1,n}, \epsilon_{2,n}, \dots, \epsilon_{K,n}\right)^T\in\R^{K}$,
\item $\underline{\bm{\eta}}_n = \left(\eta_{1,n}, \eta_{2,n}, \dots, \eta_{K,n}\right)^T\in\R^{K}$.
\end{itemize}
Let $\bx_1,\dots, \bx_K\in\R^d$, $\underline{\by}\in\R^K$, be the random variables that generate the measurements of the nodes; it is straightforward to prove that the corresponding Wiener solution, i.e.,  $\underline{\btheta}_* = \textrm{argmin}_{\underline{\btheta}} E[\|\underline{\by} - \bV^T\underline{\btheta}\|^2]$, becomes
\begin{align}
\underline{\btheta}_* = E[\bV \bV^T]^{-1}E[\bV\underline{\by}],\label{EQ:graph_Wiener}
\end{align}
provided that the autocorrelation matrix $\bm{R}=E[\bV \bV^T]$ is invertible,
where $\bV = \diag(\bz_\Omega(\bx_{1}), \bz_\Omega(\bx_{2}),\dots, \bz_\Omega(\bx_{K}))$ is a $DK\times K$ matrix that collects the transformed random variables for the whole network. Assuming that the input-output relationship of the measurements at each node follows \eqref{EQ:DKLMS_model}, the cross-correlation vector takes the form
\begin{align*}
E[\bV\underline{\by}] &= E[\bV(\bV^T\underline{\btheta}_o + \underline{\bm{\epsilon}} + \underline{\bm{\eta}})]\\
&=E[\bV\bV^T]\underline{\btheta}_o + E[\bV\underline{\bm{\epsilon}}],
\end{align*}
where for the last relation we have used that $\bm{\eta}$ is a zero mean vector representing noise and that $\bV$ and $\underline{\bm{\eta}}$ are independent. For large enough $D$, the approximation error vector $\underline{\bm{\epsilon}}$ approaches $\bZero_K$, hence the optimal solution becomes:
\begin{align*}
\underline{\btheta}_* &= E[\bV \bV^T]^{-1}\left(E[\bV\bV^T]\underline{\btheta}_o + E[\bV\underline{\bm{\epsilon}}]\right)\\
&= \underline{\btheta}_o + E[\bV \bV^T]^{-1}E[\bV\underline{\bm{\epsilon}}] \approx \underline{\btheta}_o.
\end{align*}
Here we actually imply that \eqref{EQ:DKLMS_model} can be closely approximated by $\by_n \approx \bV_n\underline{\btheta}_o + \underline{\bm{\eta}}_n$; hence, the RFF-DKLMS is actually the standard diffusion LMS applied on the data pairs $\{(\bz_\Omega(\bx_{k,n}), y_{k,n}),\; k=1,\dots,K, \ n=1,2\dots\}$. The difference is that the input vectors $\bz_\Omega(\bx_{k,n})$ may have non zero mean and do not follow, necessarily, the Gaussian distribution. Hence, the available results regarding convergence and stability of diffusion LMS (e.g., \cite{LoSa08, Cattiveli10}) cannot be applied directly (in these works the inputs are assumed to be zero mean Gaussian to simplify the formulas related to stability). To this end, we  will follow a slightly different approach. Regarding the autocorrelation matrix, we have the following result:

\begin{lemma}\label{LEM:PD}
Consider a selection of samples $\bomega_1, \bomega_2, \dots, \bomega_D$, drawn from \eqref{EQ:fourier_of_gaussian} such that $\bomega_i\not=\bomega_j$, for any $i\not=j$. Then, the matrix $\bm{R}=E[\bV \bV^T]$ is strictly positive definite (hence invertible).
\end{lemma}

\begin{proof}
Observe that the $DK\times DK$ autocorrelation matrix is given by
$\bm{R} = E[\bV \bV^T] = \diag(R_{zz}, R_{zz} \dots, R_{zz})$,
where $R_{zz} = E[\bz_{\Omega}(\bx_k)\bz_{\Omega}(\bx_k)^T]$, for all $k=1,2,\dots, K$. It suffices to prove that the $D\times D$ matrix $R_{zz}$ is strictly positive definite. Evidently, $\bc^T R_{zz} \bc = \bc^T E\left[\bz_{\Omega}(\bx_k)\bz_{\Omega}(\bx_k)^T\right] \bc  = E\left[\left(\bz_{\Omega}(\bx_k)^T \bc\right)^2\right] \geq 0$, for all $\bc\in\R^D$. Now, assume that there is a $\bc\in\R^D$ such that $E\left[\left(\bz_{\Omega}(\bx_k)^T \bc\right)^2\right] = 0$. Then $\bz_{\Omega}(\bx)^T \bc=0$ for all $\bx\in \R^D$, or equivalently, $\sum_{i=1}^D c_i \cos(\bomega_i^T \bx +b_i) = 0$, for all $\bx\in \R^D$. Thus, $\bc = \bZero$.
\end{proof}

\noindent As expected, the eigenvalues of $R_{zz}$ play a pivotal role in the convergence's study of the algorithm. As $R_{zz}$ is a strictly positive definite matrix, its eigenvalues satisfy $0<\lambda_1\leq\lambda_2\leq\dots\leq\lambda_D$.

\begin{proposition}\label{PRO:DKLMS_cons}
If the the step update $\mu$ satisfies: $0<\mu<\frac{2}{\lambda_D}$, where $\lambda_D$ is the maximum eigenvalue of $R_{zz}$, then the RFF-DKLMS achieves asymptotic consensus in the mean, i.e.,
\begin{align*}
\lim_{n}E[\btheta_{k,n} - \btheta_o] = \bZero_D, \textrm{ for all } k=1,2,\dots, K.
\end{align*}
\end{proposition}
\begin{proof}
See Appendix \ref{SEC:proof_DKLMS_cons}.
\end{proof}

\begin{remark}If $\bx_{k,n}\sim \cN(\bzero, \sigma_X \bI_d)$, it is possible to evaluate explicitly the entries of $R_{zz}$, i.e.,
\begin{align*}
r_{i,j} =& \frac{1}{2} \exp\left(\frac{-\|\bomega_i - \bomega_j\|^2\sigma_X^2}{2}\right)\cos(b_i - b_j) \\
&+ \frac{1}{2}\exp\left(\frac{-\|\bomega_i + \bomega_j\|^2\sigma_X^2}{2}\right)\cos(b_i + b_j).
\end{align*}
\end{remark}

\begin{proposition}\label{PRO:DKLMS_stab}
For stability in the mean-square sense, we must ensure that both $\mu$ and $A$ satisfy:
\begin{align*}
\left|\rho\left(\bm{I}_{D^2K^2} - \mu \left(\bm{R}\boxtimes\bm{I}_{DK}+\bm{I}_{DK}\boxtimes\bm{R}\right) (\bm{A}\boxtimes\bm{A})\right)\right| < 1,
\end{align*}
where $\boxtimes$ denotes the unbalanced block Kronecker product.
\end{proposition}

\begin{proof}
See Appendix \ref{SEC:proof_DKLMS_stab}.
\end{proof}

In the following, we present some experiments to illustrate the performance of the proposed scheme. We demonstrate that the estimation provided by the cooperative strategy is better than having each node working alone (i.e., lower MSE). Similar to section \ref{SEC:DKLMS}, each realization of the experiments uses a different random connected graph with $M=20$ nodes and probability of attachment per node equal to 0.2. The adjacency matrix, $A$, of each graph was generated using the Metropolis rule (resulting to graphs with mean algebraic connectivity around $0.69$), while for the non-cooperative strategies, we used a graph that connects each node to itself, i.e., $A = I_{20}$.  All parameters were optimized (after trials) to give the lowest MSE. The algorithms were implemented in MatLab and the experiments were performed on a i7-3770 machine running at $3.4$GHz with 32 Mb of RAM.

\subsubsection{Example 1. A Linear Expansion in terms of kernels}\label{EXP:Graph_kernel_expansion}
In this set-up, we generate $5000$ data pairs for each node using the following model:
$y_{k,n} = \sum_{m=1}^M a_m\kappa(\bc_{m}, \bx_{k,n}) + \eta_{k,n}$,
where $\bx_{k,n}\in\R^5$ are drawn from $\cN(\bZero, I_5)$ and the noise are i.i.d. Gaussian samples with $\sigma_\eta=0.1$. The parameters of the expansion (i.e., $a_1,\dots, a_M$) are drawn from $\cN(0, 25)$, the kernel parameter $\sigma$ is set to $5$, the step update to $\mu=1$ and the number of random Fourier features to $D=2500$. Figure \ref{FIG:RFF-DKLMS}(a) shows the evolution of the MSE over all network nodes for 100 realizations of the experiment. We note that the selected value of step size satisfies the conditions of proposition \ref{PRO:DKLMS_cons}.

\subsubsection{Example 2}\label{SEC:Graph_square}
Next, we generate the data pairs for each node using the following simple non-linear model:
$y_{k,n} = \bw_0^T\bx_{k,n} + 0.1\cdot(\bw_1^T\bx_{k,n})^2 + \eta_{k,n}$,
where $\eta_{k,n}$ represent zero-mean i.i.d. Gaussian noise with $\sigma_\eta=0.05$ and the coefficients of the vectors $\bw_0, \bw_1\in\R^5$ are i.i.d. samples drawn from $\cN(0,1)$. Similarly to Example 1, the kernel parameter $\sigma$ is set to $5$ and the step update to $\mu=1$. The number of random Fourier coefficients for RFFKLMS was set to $D=300$.
Figure \ref{FIG:RFFKLMS}(b) shows the evolution of the MSE over all network nodes for 1000 realizations of the experiment over $15000$ samples.

\subsubsection{Example 3}\label{SEC:par1}
Here we adopt the following chaotic series model \cite{Stochastic_KLMS}:
$d_{k,n} = \frac{d_{k,n-1}}{1+d^2_{k,n-1}} + u^3_{k,n-1}$,
$y_{k,n} = d_{k,n} + \eta_{k,n}$,
where $\eta_n$ is zero-mean i.i.d. Gaussian noise with $\sigma_\eta=0.01$ and $u_n$ is also zero-mean i.i.d. Gaussian with $\sigma_u = 0.15$. The kernel parameter $\sigma$ is set to $0.05$, the number of Fourier features to $D=100$ and the step update to $\mu=1$. We have also initialized $d_1$ to $1$.
Figure \ref{FIG:RFF-DKLMS}(c) shows the evolution of the MSE over all network nodes for 1000 realizations of the experiment over 500 samples.

\subsubsection{Example 4}\label{SEC:par2}
For the final example, we use another chaotic series model \cite{Stochastic_KLMS}:
$d_{k,n} = u_{k,n} + 0.5v_{k,n} - 0.2 d_{k,n-1} + 0.35 d_{k,n-2}$, $y_{k,n} = \phi(d_{k,n}) + \eta_{k,n}$,
\begin{align*}
\phi(d_{k,n}) =& \left\{\begin{matrix} \frac{d_{k,n}}{3(0.1 + 0.9 d_{k,n}^2)^{1/2}} & d_{k,n}\geq 0 \cr
 \frac{-d_{k,n}^2(1-\exp(0.7d_{k,n}))}{3} & d_{k,n}<0 \end{matrix} \right.,
\end{align*}
where $\eta_{k,n}$, $v_{k,n}$  are zero-mean i.i.d. Gaussian noise with $\sigma_\eta=0.001$ and $\sigma_v^2 = 0.0156$ respectively, and $u_{k,n} = 0.5v_{k,n} + \hat\eta_{k,n}$, where $\hat\eta_n$ is also i.i.d. Gaussian with $\sigma^2 = 0.0156$. The kernel parameter $\sigma$ is set to $0.05$ and the step update to $\mu=1$. We have also initialized $d_1, d_2$ to $1$.
Figure \ref{FIG:RFFKLMS}(d) shows the evolution of the MSE over all network nodes for $1000$ realizations of the experiment over $1000$ samples. The number of random Fourier features was set to $D=200$.

\begin{figure}[t]

\begin{minipage}[b]{.48\linewidth}
  \centering
  \centerline{\epsfig{figure=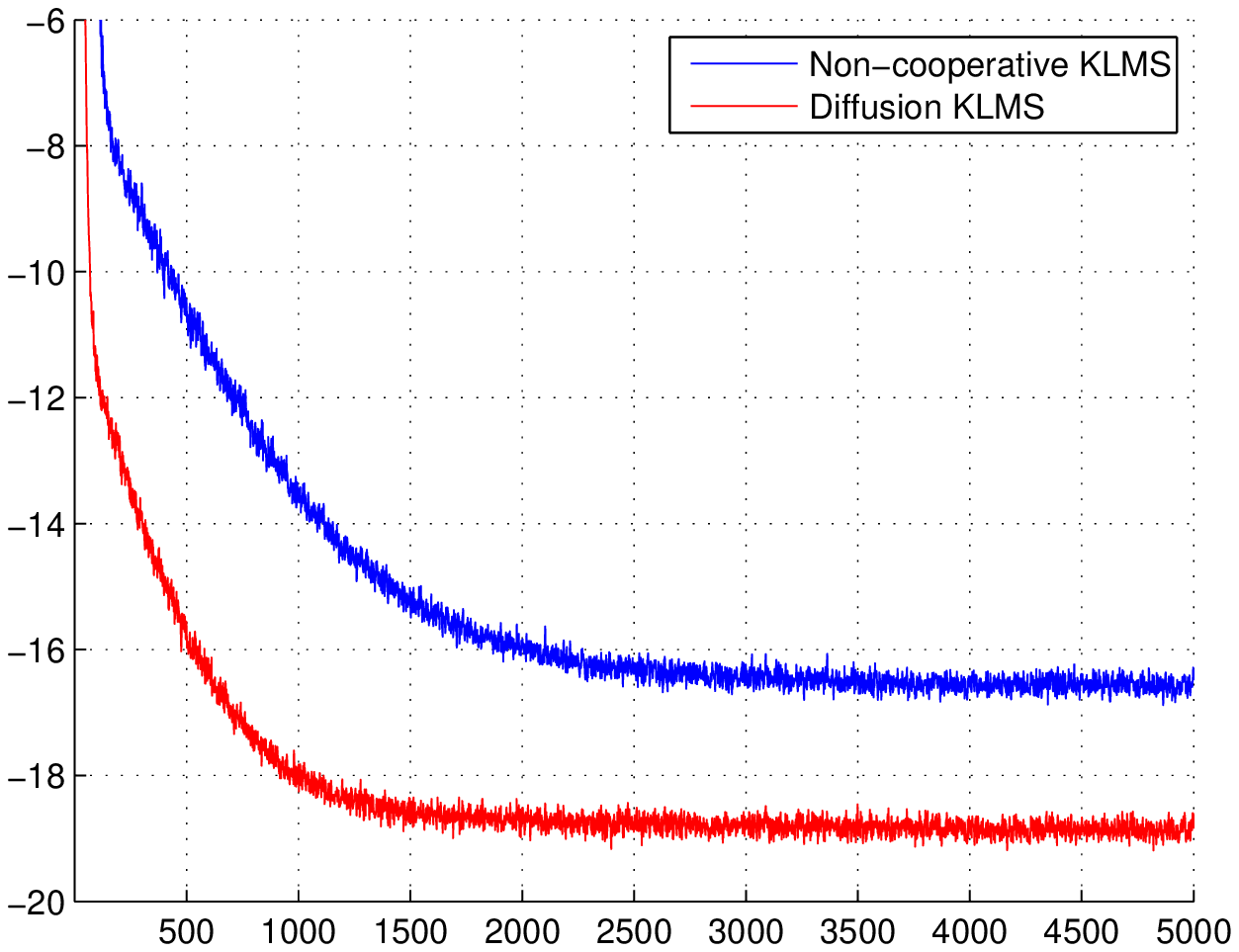,width=3.5cm}}
  \centerline{\scriptsize{(a) Example 1}}\medskip
\end{minipage}
\hfill
\begin{minipage}[b]{0.48\linewidth}
  \centering
  \centerline{\epsfig{figure=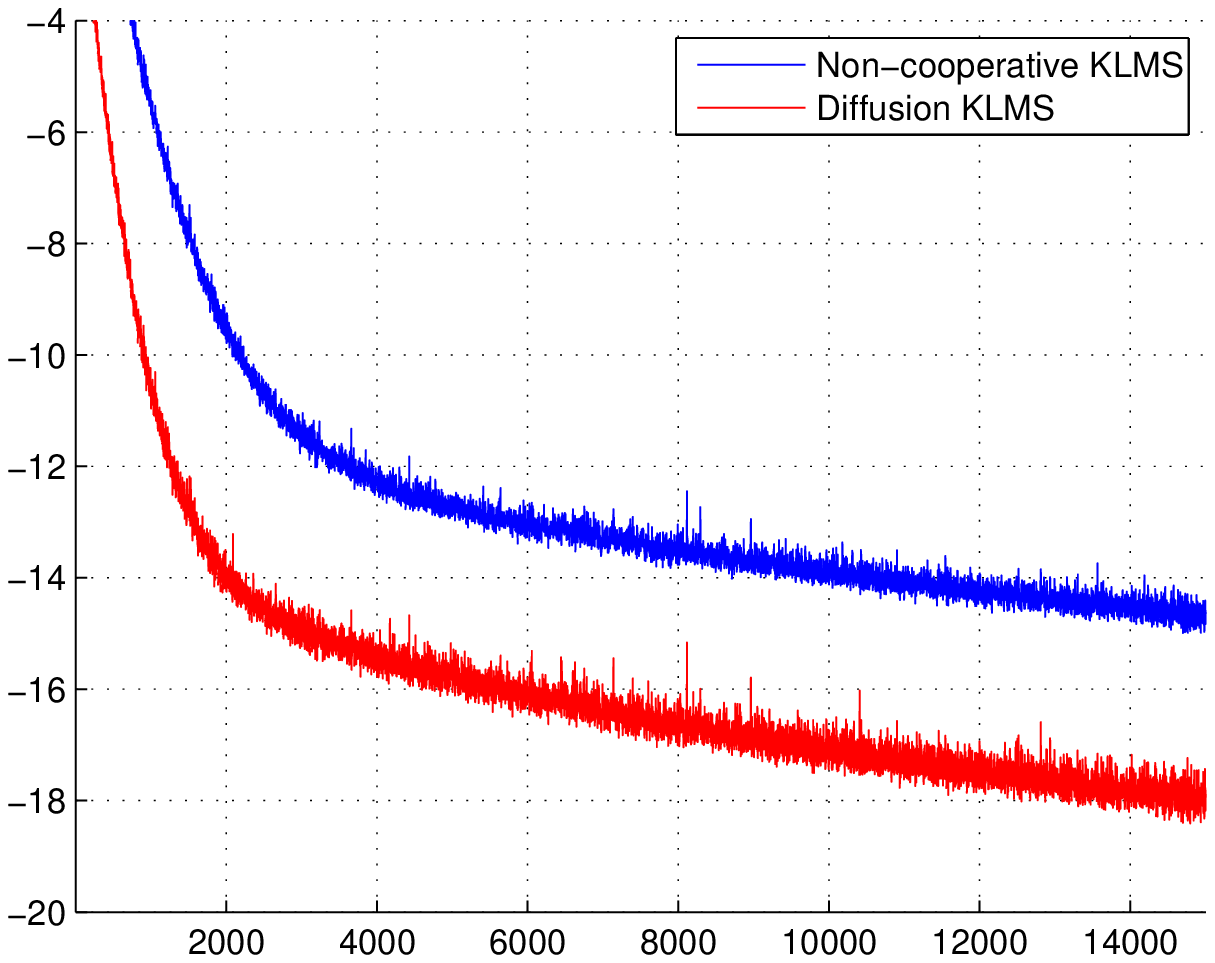,width=3.5cm}}
  \centerline{\scriptsize{(b) Example 2}}\medskip
\end{minipage}

\begin{minipage}[b]{.48\linewidth}
  \centering
  \centerline{\epsfig{figure=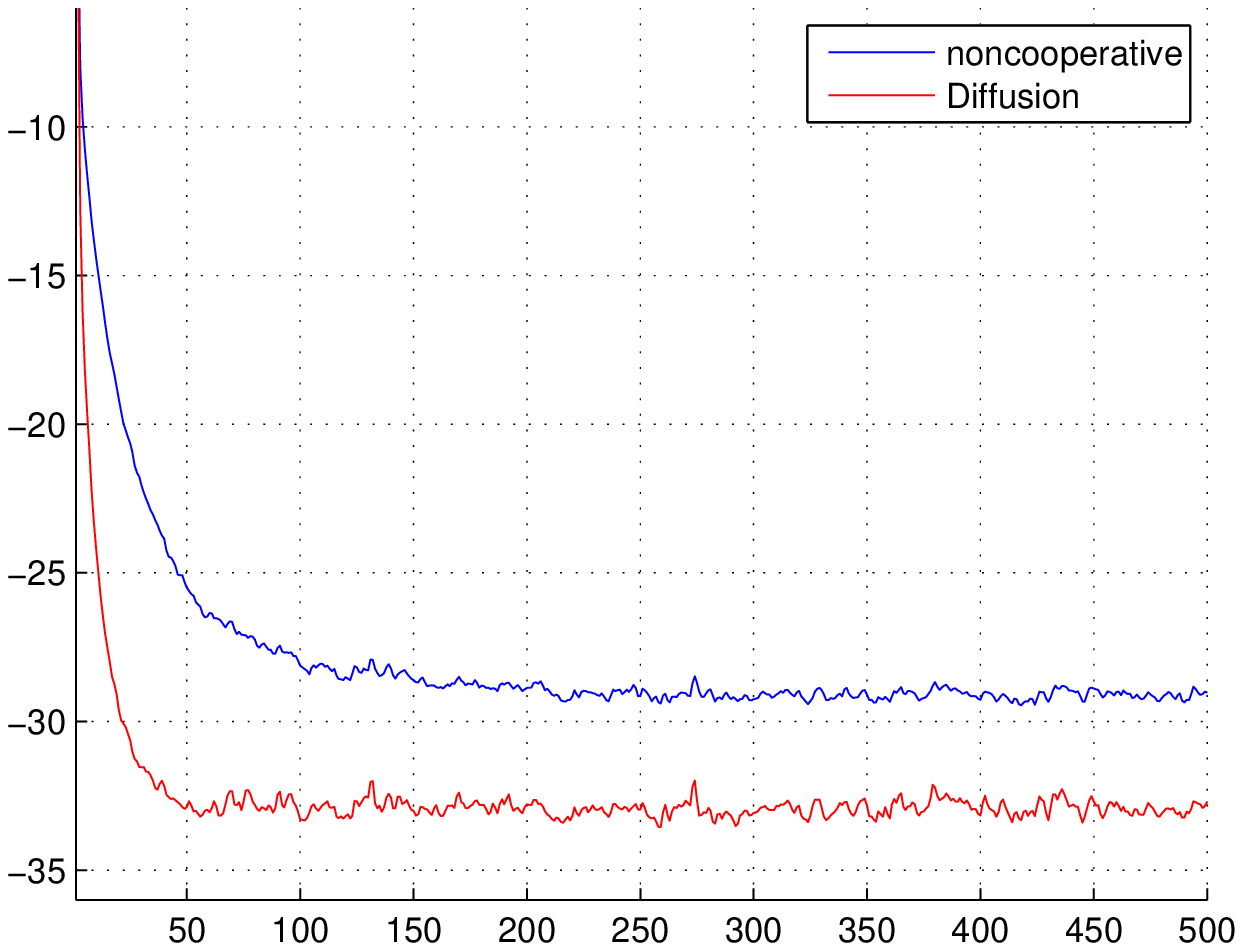,width=3.5cm}}
  \centerline{\scriptsize{(c) Example 3}}\medskip
\end{minipage}
\hfill
\begin{minipage}[b]{0.48\linewidth}
  \centering
  \centerline{\epsfig{figure=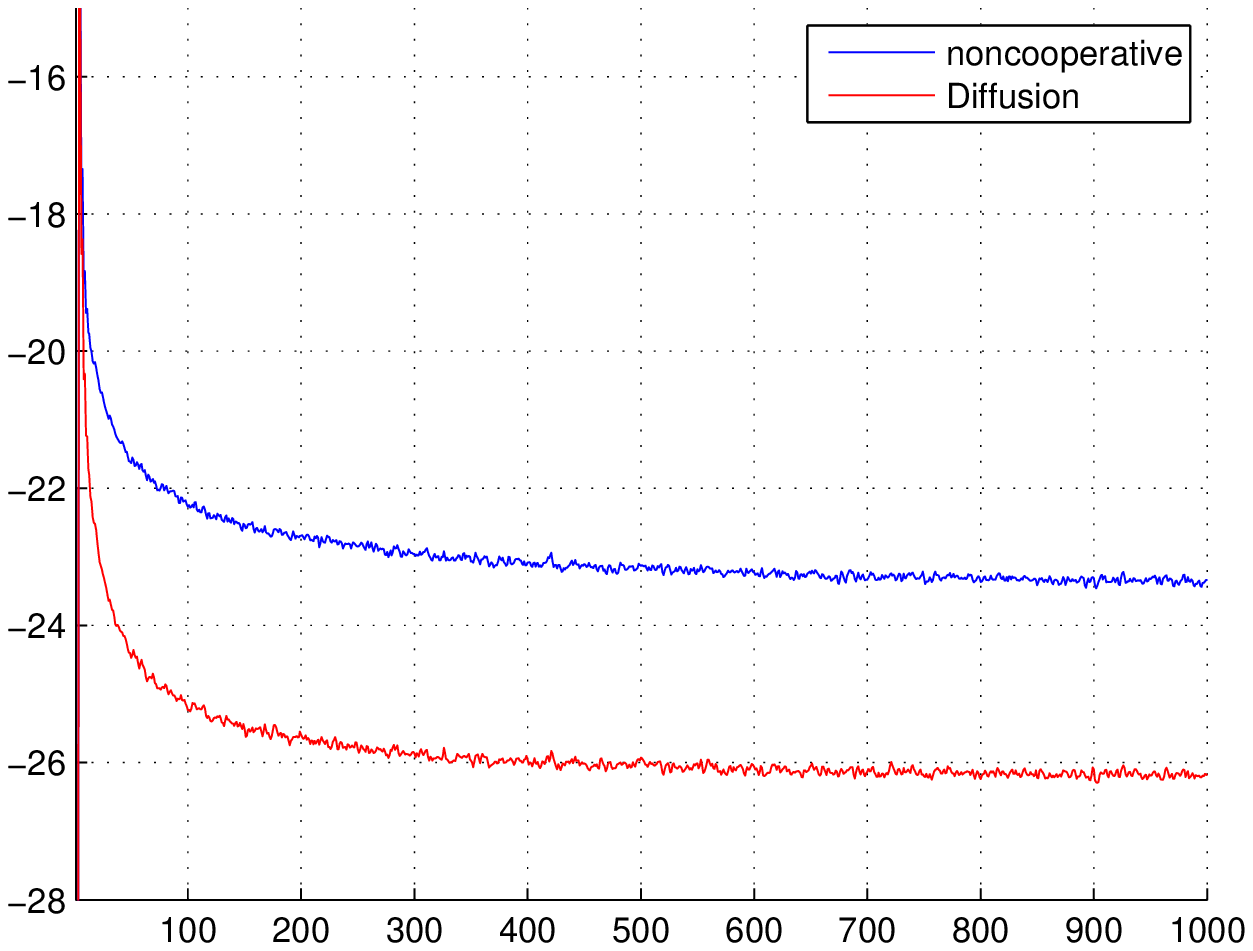,width=3.5cm}}
  \centerline{\scriptsize{(d) Example 4}}\medskip
\end{minipage}
\caption{Comparing the performances of RFF Diffusion KLMS versus the non-cooperative strategy.}
\label{FIG:RFF-DKLMS}
\end{figure}

\section{Revisiting Online Kernel based learning}\label{SEC:OKL}
In this section, we investigate the use of random Fourier features as a general framework for online kernel-based learning. The framework presented here can be seen as a special case of the general distributed method presented in section \ref{SEC:Distributed} for a network with a single node.
Similar to the case of the standard KLMS, the learning algorithms considered here adopt a gradient descent rationale to minimize a specific loss function, $\cL(\bx,y,f)$ for $f\in\cH$, so that $f$ approximates the relationship between $\bx$ and $y$, where $\cH$ is the RKHS induced by a specific choice of a shift invariant (semi)positive definite kernel, $\kappa$. Hence, in general, these algorithms can be summarized by the following step update equation:  $f_n = f_{n-1} + \mu_n \nabla_f\cL(\bx_n,y_n,f_{n-1})$. 
Algorithm \ref{Alg:RFF-OKL} summarizes the proposed procedure for online kernel-based learning. The performance of the algorithm  depends on the quality of the adopted approximation. Hence, a sufficiently large $D$ has to be selected.

\begin{algorithm}[t]
\caption{Random Fourier Features Online Kernel-based Learning (RFF-OKL).}\label{Alg:RFF-OKL}
\begin{algorithmic}
\State $D = \{(\bx_n, y_n), n=1,2,\dots\}$ \Comment{Input}
\State Select a specific (semi)positive definite kernel, a specific loss function $\cL$ and a sequence of possible variable learning rates $\mu_n$. Then generate the matrix $\Omega$ as in \eqref{EQ:omega}.
\State $\btheta_0 \gets \bZero_D$ \Comment{Initialization}
\For{$n=1,2,3, ...$}
\State  $\btheta_n = \btheta_{n-1} + \mu_n \nabla_{\btheta}\cL(\bx_n,y_n,\btheta_{n-1})$. \Comment{Step update}
\EndFor
\end{algorithmic}
\end{algorithm}

Although algorithm \ref{Alg:RFF-OKL} is given in a general setting, in the following we focus on the fixed-budget KLMS.
As it has been discussed in section \ref{SEC:prelim}, KLMS adopts the MSE cost function, which in the proposed framework takes the form: $\cL(\bx,y,\btheta) = E[(y_n - \btheta^T\bz_\Omega(\bx_n))^2]$. Hence, the respective step update equation of algorithm \ref{Alg:RFF-OKL} becomes
\begin{align}\label{EQ:KLMS}
\btheta_n = \btheta_{n-1} + \mu \varepsilon_n\bz_\Omega(\bx_n),
\end{align}
where $\varepsilon_n = y_n - \btheta_{n-1}^T\bz_\Omega(\bx_n)$.
Observe that, contrary to the typical implementations of KLMS, where the system's output is a growing expansion of kernel functions and hence special care has to be carried out to prune the so called dictionary, the proposed approach employs a fixed-budget rationale, which doesn't require any further treatment. We call this scheme the Random Fourier Features KLMS (RFF-KLMS) \cite{RFFKLMS1, RFFKLMS2}.

The study of the convergence properties of RFFKLMS is based on those of the standard LMS. Henceforth, we will assume that the data pairs are generated  by
\begin{align}
y_n = \sum_{m=1}^M a_m\kappa(\bc_m, \bx_n) + \eta_n,\label{EQ:input_model}
\end{align}
where $\bc_1,\dots,\bc_M$ are fixed centers, $\bx_n$ are zero-mean i.i.d, samples drawn from the Gaussian distribution with covariance matrix $\sigma_x^2\bI_d$ and $\eta_n$ are i.i.d. noise samples drawn from $\cN(0, \sigma_\eta^2)$.
Similar to the diffusion case, the eigenvalues of $R_{zz}$, i.e., $0<\lambda_1\leq\lambda_2\leq\dots\leq\lambda_D$, play a pivotal role in the convergence's study of the algorithm. Applying similar assumptions as in the case of the standard LMS (e.g., independence between $\bx_n, \bx_m$, for $n\neq m$ and between $\bx_n, \eta_n$), we can prove the following results.
\begin{proposition}\label{THE:convergence}
For datasets generated by \eqref{EQ:input_model} we have:
\begin{enumerate}
\item If $0<\mu<2/\lambda_D$, then  RFFKLMS converges in the mean, i.e., $E[\btheta_n - \btheta_{\textrm{o}}]\rightarrow \bZero$.
\item The optimal MSE  is given by $$J_n^{\textrm{opt}} = \sigma_\eta^2 + E[\epsilon_n]  - E[\epsilon_n\bz_{\Omega}(\bx_n)]R_{zz}^{-1}E[\epsilon_n\bz_{\Omega}(\bx_n)^T].$$
For large enough $D$, we have $J_n^{\textrm{opt}} \approx \sigma_\eta^2$.
\item The excess MSE is given by $J_n^{\textrm{ex}} = J_n - J_n^{\textrm{opt}}=\trace\left(R_{zz}A_n\right)$, where $A_n = E[(\btheta_n - \btheta_{\textrm{o}})(\btheta_n - \btheta_{\textrm{o}})^T]$.
\item If $0<\mu<1/\lambda_D$, then  $A_n$ converges. For large enough $n$ and $D$ we can approximate $A_n$'s evolution as $A_{n+1} \approx A_n - \mu\left(R_{zz}A_n + A_n R_{zz}\right) + \mu^2\sigma_\eta^2 R_{zz}$. Using this model we can approximate the steady-state MSE ($\approx \trace\left(R_{zz}A_n\right) + \sigma_\eta^2$).
\end{enumerate}
\end{proposition}

\begin{proof}
The proofs use standard arguments as in the case of the standard LMS. Hence we do not provide full details. The reader is addressed to any LMS textbook.\\
1) See Proposition \ref{PRO:DKLMS_cons}.\\
2) Replacing $\btheta_n$ with $\btheta_{\textrm{o}}$ in $J_n = E[\varepsilon_n^2]$ gives the result. For large enough $D$, $\epsilon_n$ is almost zero, hence we have $J_n^{\textrm{opt}} \approx \sigma_\eta^2$.\\
3) Here, we use the additional assumptions that $\bv_n$ is independent of $\bx_n$ and that $\epsilon_n$ is independent of $\eta_n$. The result follows after replacing $J_n$ and $J_n^{\textrm{opt}}$ and performing simple algebra calculations.\\
4) Replacing $\btheta_{\textrm{o}}$ and dropping out the terms that contain the term $\epsilon_n$, the result is obtained.
\end{proof}

\begin{remark}
Observe that, while the first two results can be regarded as special cases of the distributed case (see proposition \ref{PRO:DKLMS_cons} and the related discussion in section \ref{SEC:Distributed}), the two last ones describe more accurately the evolution of the solution in terms of mean square stability, than the one given in proposition \ref{PRO:DKLMS_stab}, for the general distributed scheme (where no formula for $\bm{B}_n$ is given). This becomes possible because the related formulas take a much simpler form, if the graph structure is reduced to a single node.
\end{remark}

In order to illustrate the performance of the proposed algorithm and compare its behavior to the other variants of KLMS, we also present some related simulations. We choose the QKLMS \cite{QKLMS} as a reference, since this is one of the most effective and fast KLMS pruning methods. In all experiments, that are presented in this section (described below), we use the same kernel parameter, i.e., $\sigma$, for both RFFKLMS and QKLMS as well as the same step-update parameter $\mu$. The quantization parameter $q$ of the QKLMS controls the size of the dictionary. If this is too large, then the dictionary will be small and the achieved MSE at steady state will be large. Typically, however, there is a value for $q$ for which the best possible MSE (which is very close to the MSE of the unsparsified version) is attained at steady state, while any smaller quantization sizes provide negligible improvements (albeit at significantly increased complexity). In all experimental set-ups, we tuned $q$ (using multiple trials) so that it leads to the best performance. On the other hand, the performance of RFFKLMS depends largely on $D$, which controls the quality of the kernel approximation. Similar to the case of QKLMS, there is a value for $D$ so that RFFKLMS attains its lowest steady-state MSE, while larger values provide negligible improvements. For our experiments, the chosen values for $q$ and $D$ provide results so that to trace out the results provided by the original (unsparsified) KLMS.  Table \ref{TAB:KLMS_times} gives the mean training times for QKLMS and RFFKLMS on the same i7-3770 machine using a MatLab implementation (both algorithms were optimized for speed). We note that the complexity of the RFFKLMS is $\mathcal{O}(Dd)$, while the complexity of QKLMS is $\mathcal{O}(Md)$. Our experiments indicate that in order to obtain similar error floors, the required complexity of RFFKLMS is lower than that of QKLMS.

\subsubsection{Example 5. A Linear expansion in terms of Kernels}
Similar to example 1 in section \ref{SEC:DKLMS}, we generate $5000$ data pairs using \eqref{EQ:input_model} and the same parameters (for only one node). Figure \ref{FIG:kernel_expansion} shows the evolution of the MSE for $500$ realizations of the experiment over different values of $D$. The algorithm reaches steady-state around $n=3000$. The attained MSE is getting closer to the approximation given in proposition \ref{THE:convergence} (dashed line in the figure) as $D$ increases. Figure \ref{FIG:RFFKLMS}(a) compares the performances of RFFKLMS and QKLMS for this particular set-up for $500$ realizations of the experiment using $8000$ data pairs. The quantization size of QKLMS was set to $q=5$ and the number of Fourier features for the RFFKLMS was set to $D=2500$.

\begin{figure}[t]

\begin{minipage}[b]{.48\linewidth}
  \centering
  \centerline{\epsfig{figure=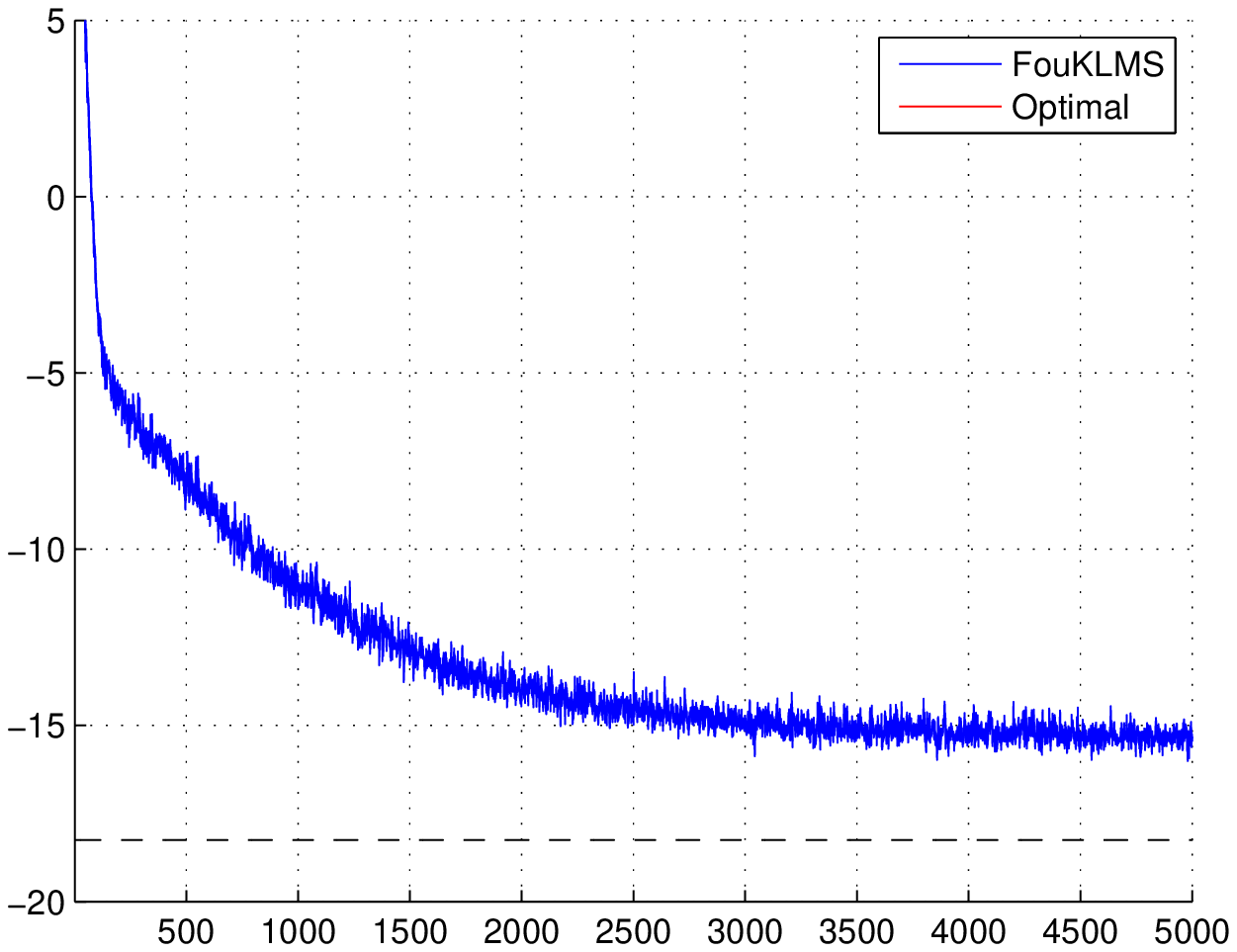,width=3.5cm}}
  \centerline{\scriptsize{(a) $D=500$}}\medskip
\end{minipage}
\hfill
\begin{minipage}[b]{0.48\linewidth}
  \centering
  \centerline{\epsfig{figure=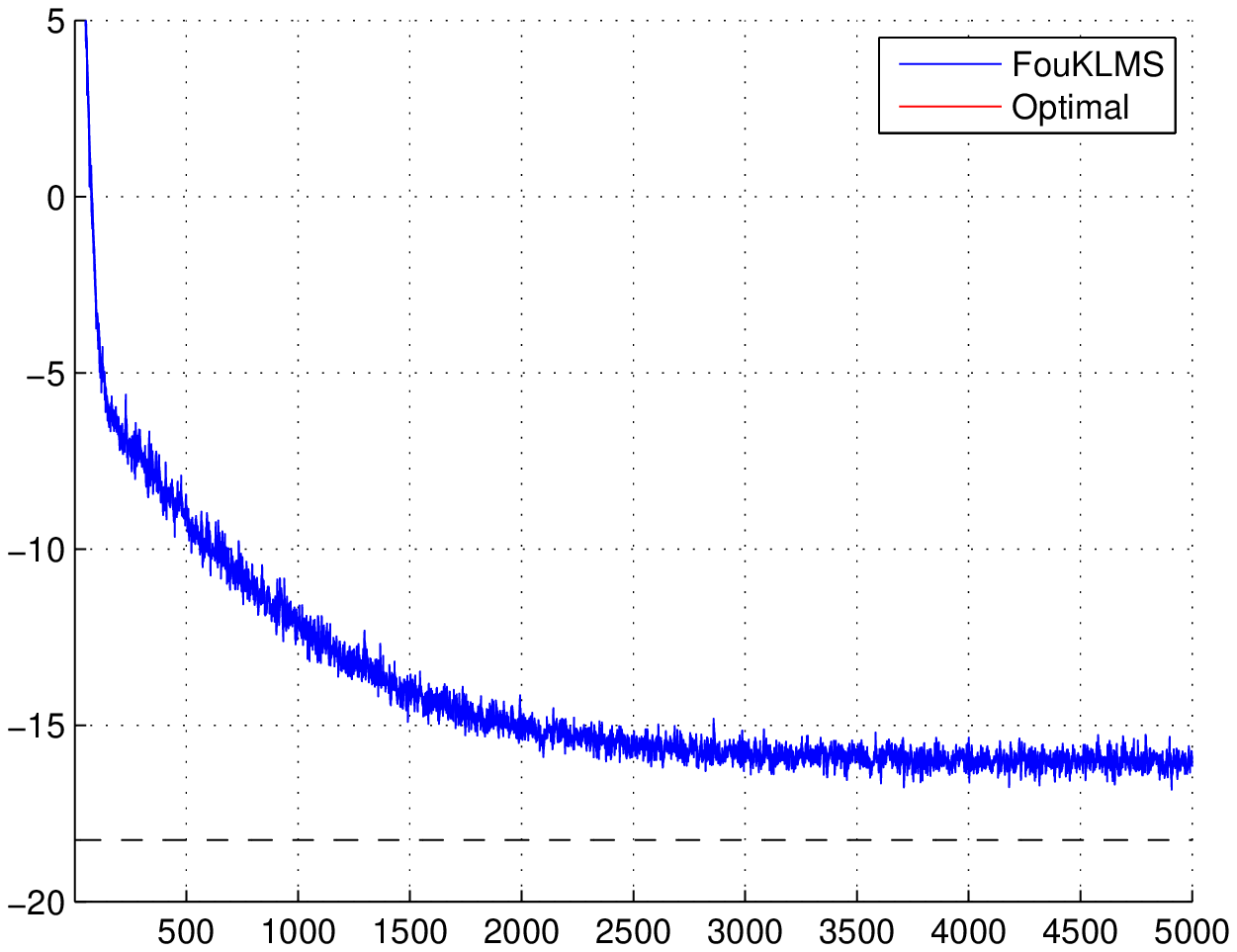,width=3.5cm}}
  \centerline{\scriptsize{(b) $D=1000$}}\medskip
\end{minipage}

\begin{minipage}[b]{.48\linewidth}
  \centering
  \centerline{\epsfig{figure=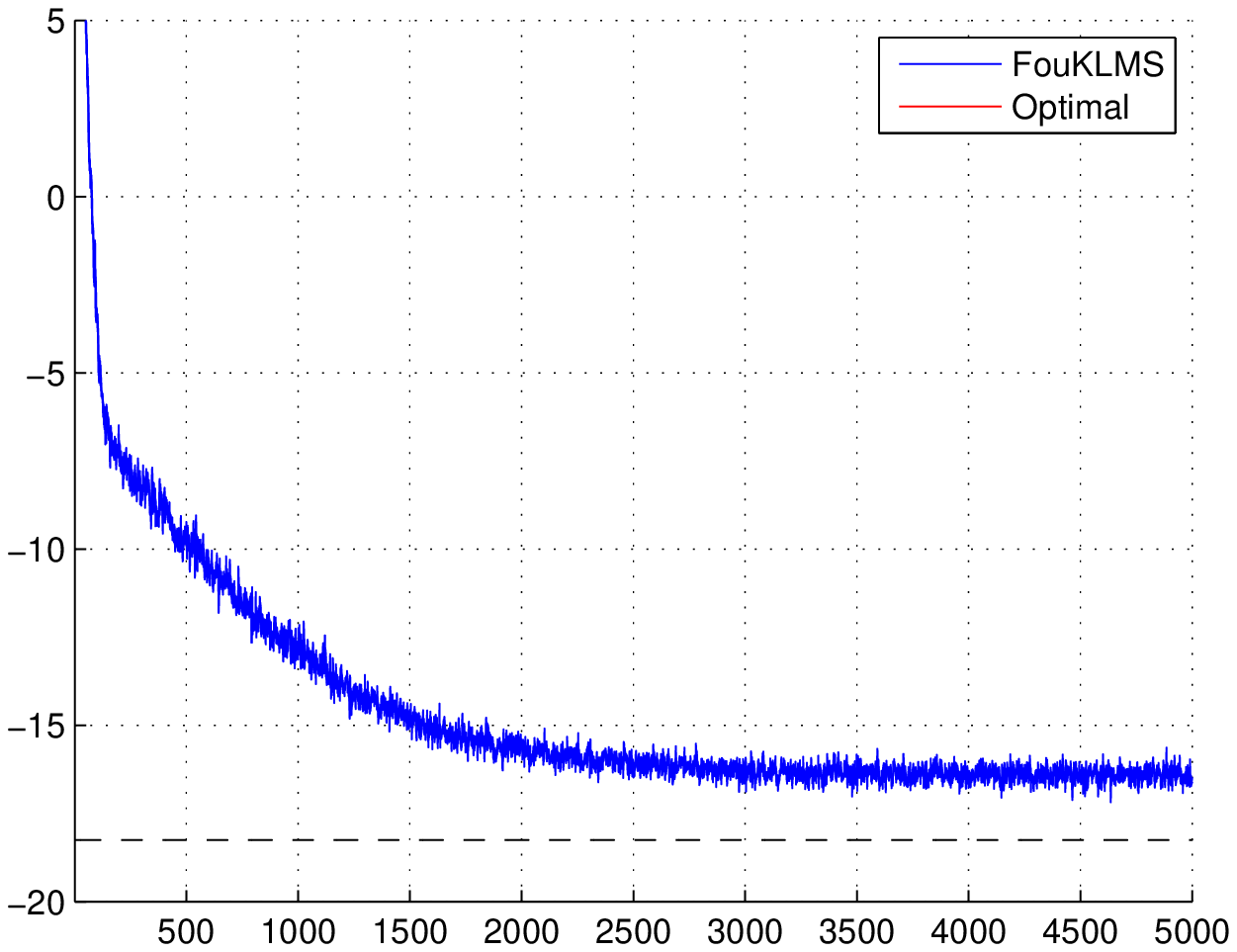,width=3.5cm}}
  \centerline{\scriptsize{(c) $D=2000$}}\medskip
\end{minipage}
\hfill
\begin{minipage}[b]{0.48\linewidth}
  \centering
  \centerline{\epsfig{figure=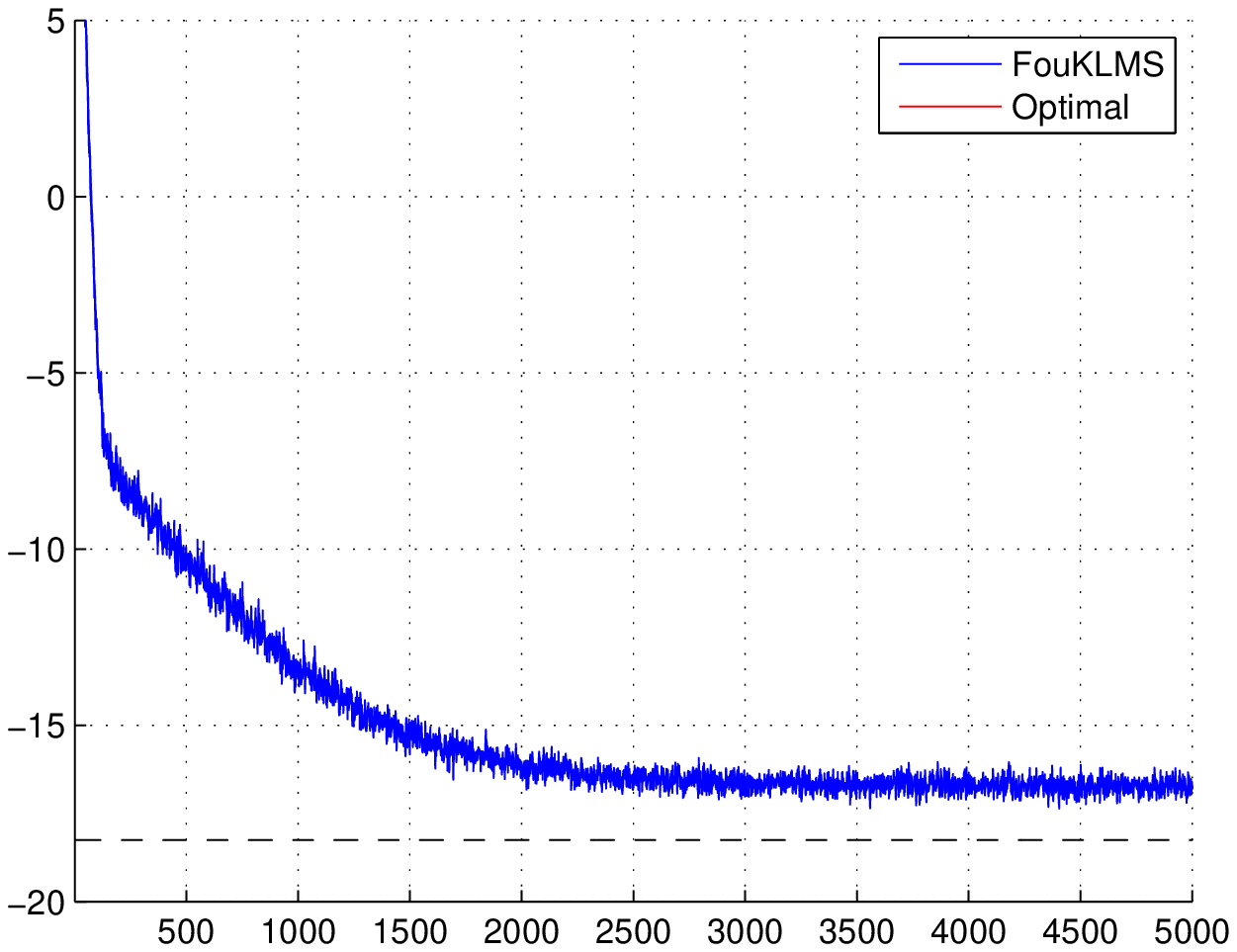,width=3.5cm}}
  \centerline{\scriptsize{(d) $D=5000$}}\medskip
\end{minipage}
\caption{Simulations of RFFKLMS (with various values of $D$) applied on data pairs generated by \eqref{EQ:input_model}. The results are averaged over $500$ runs. The horizontal dashed line in the figure represents the approximation of the steady-state MSE given in theorem \ref{THE:convergence}.}
\label{FIG:kernel_expansion}
\end{figure}

\subsubsection{Example 6}\label{SEC:square}
Next, we use the same non-linear model as in example 2 of section \ref{SEC:Distributed}, i.e.,
$y_n = \bw_0^T\bx_n + 0.1\cdot(\bw_1^T\bx_n)^2 + \eta_n$. The parameters of the model and the RFF-KLMS are the same as in example 1. The quantization size of the QKLMS was set to $q=5$, leading to an average dictionary size $M=100$.
Figure \ref{FIG:RFFKLMS}(b) shows the evolution of the MSE for both QKLMS and RFFKLMS running $1000$ realizations of the experiment over $15000$ samples.

\subsubsection{Example 7}\label{SEC:par1}
Here we adopt the same chaotic series model as in example 3 of section \ref{SEC:DKLMS}, with the same parameters.
Figure \ref{FIG:RFFKLMS}(c) shows the evolution of the MSE for both QKLMS and RFFKLMS running $1000$ realizations of the experiment over $500$ samples. The quantization parameter $q$ for the QKLMS was set to $q=0.01$, leading to an average dictionary size $M=7$.

\subsubsection{Example 8}\label{SEC:par2}
For the final example, we use the chaotic series model of example 4 in section \ref{SEC:DKLMS} with the same parameters.
Figure \ref{FIG:RFFKLMS}(d) shows the evolution of the MSE for both QKLMS and RFFKLMS running $1000$ realizations of the experiment over $1000$ samples. The parameter $q$ was set to $q=0.01$, leading to $M=32$.

\begin{figure}[t]

\begin{minipage}[b]{.48\linewidth}
  \centering
  \centerline{\epsfig{figure=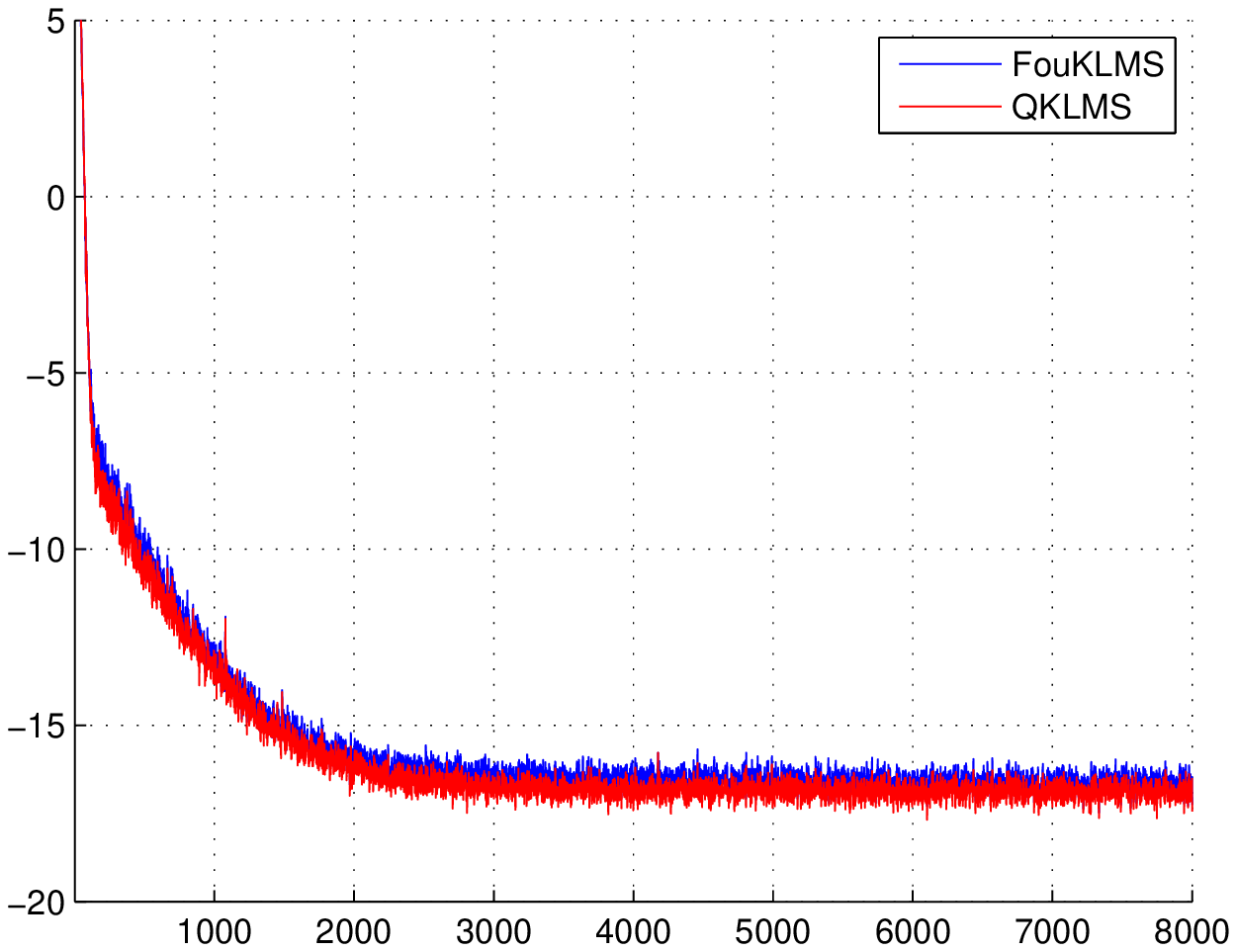,width=3.5cm}}
  \centerline{\scriptsize{(a) Example 5}}\medskip
\end{minipage}
\hfill
\begin{minipage}[b]{0.48\linewidth}
  \centering
  \centerline{\epsfig{figure=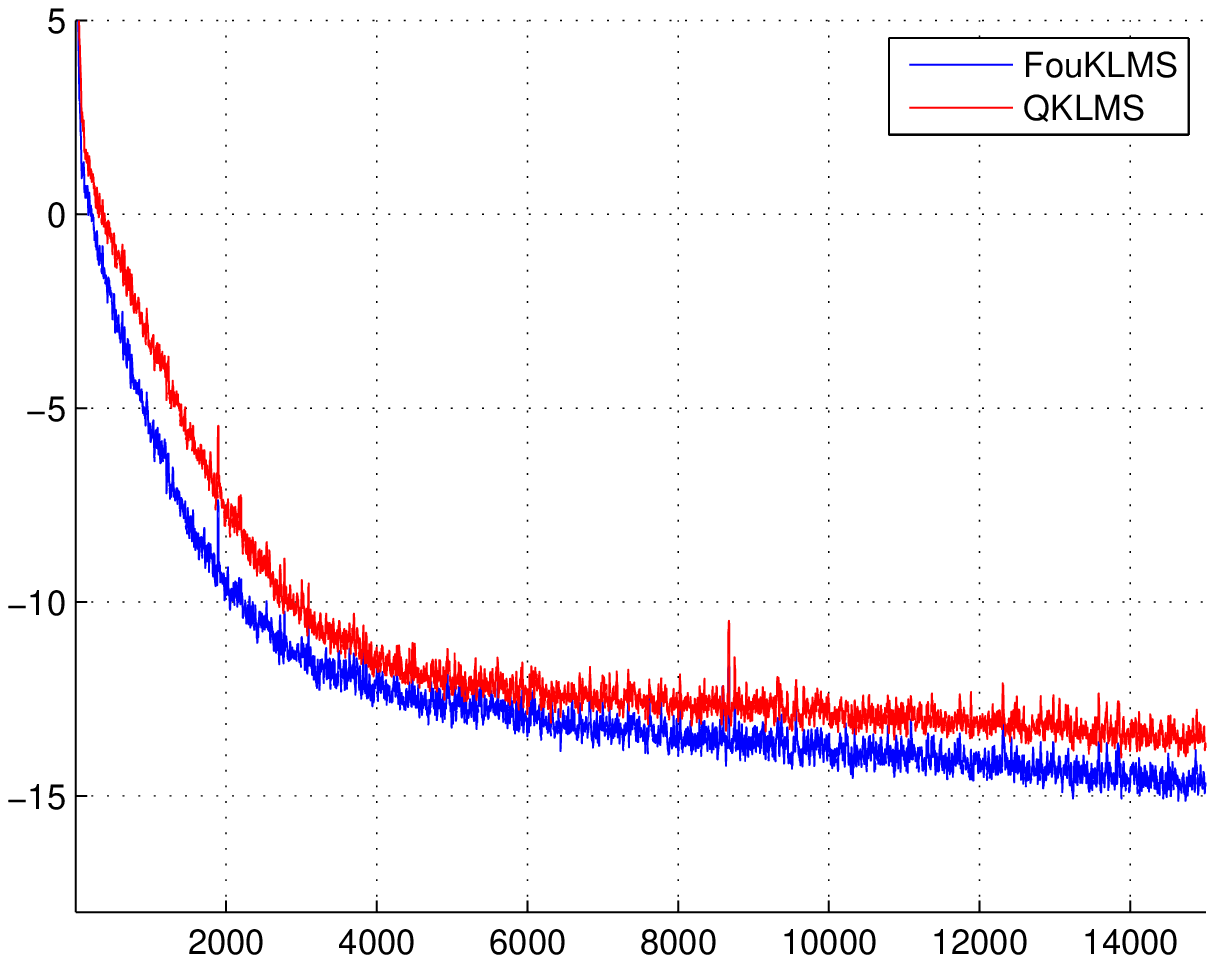,width=3.5cm}}
  \centerline{\scriptsize{(b) Example 6}}\medskip
\end{minipage}

\begin{minipage}[b]{.48\linewidth}
  \centering
  \centerline{\epsfig{figure=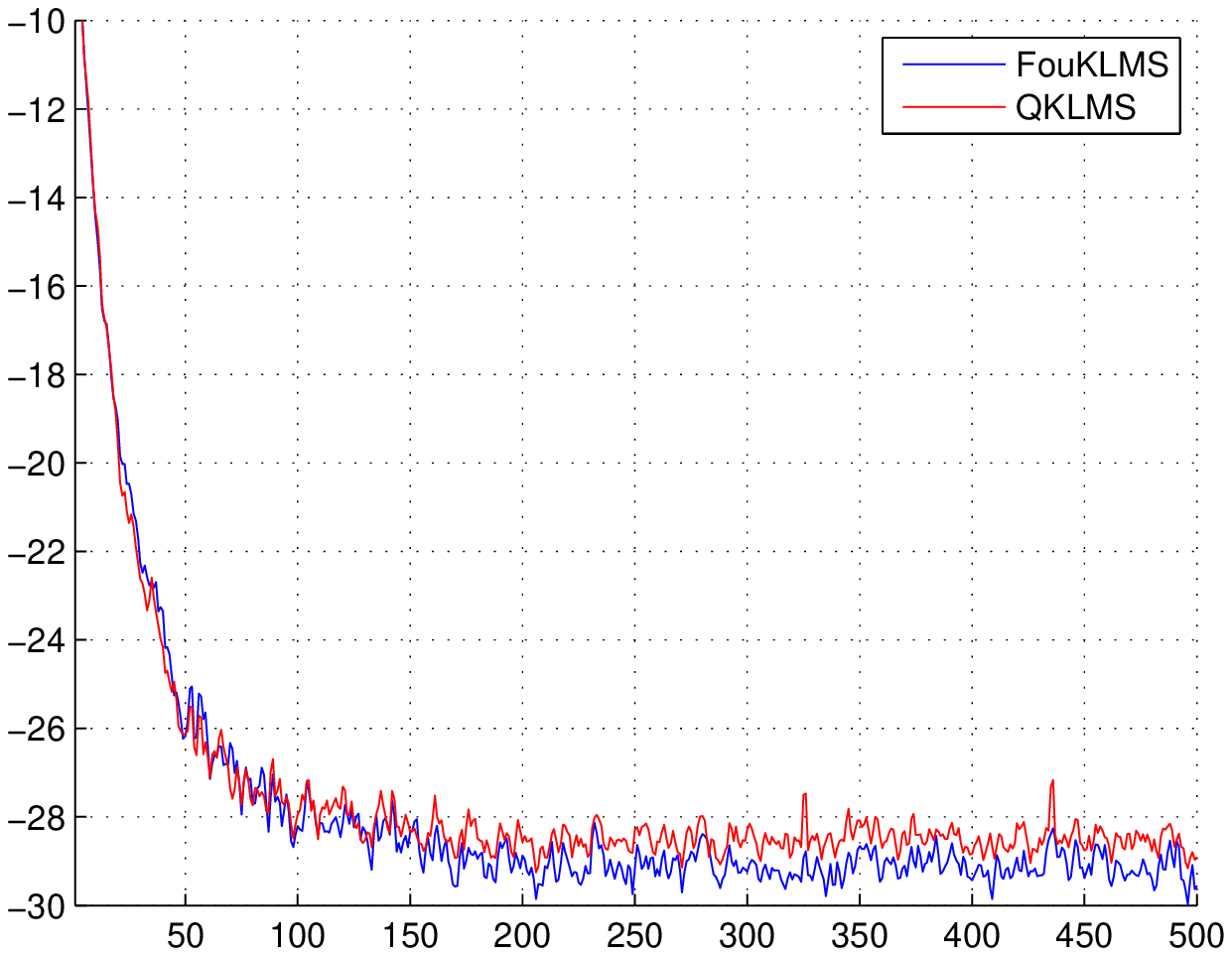,width=3.5cm}}
  \centerline{\scriptsize{(c) Example 7}}\medskip
\end{minipage}
\hfill
\begin{minipage}[b]{0.48\linewidth}
  \centering
  \centerline{\epsfig{figure=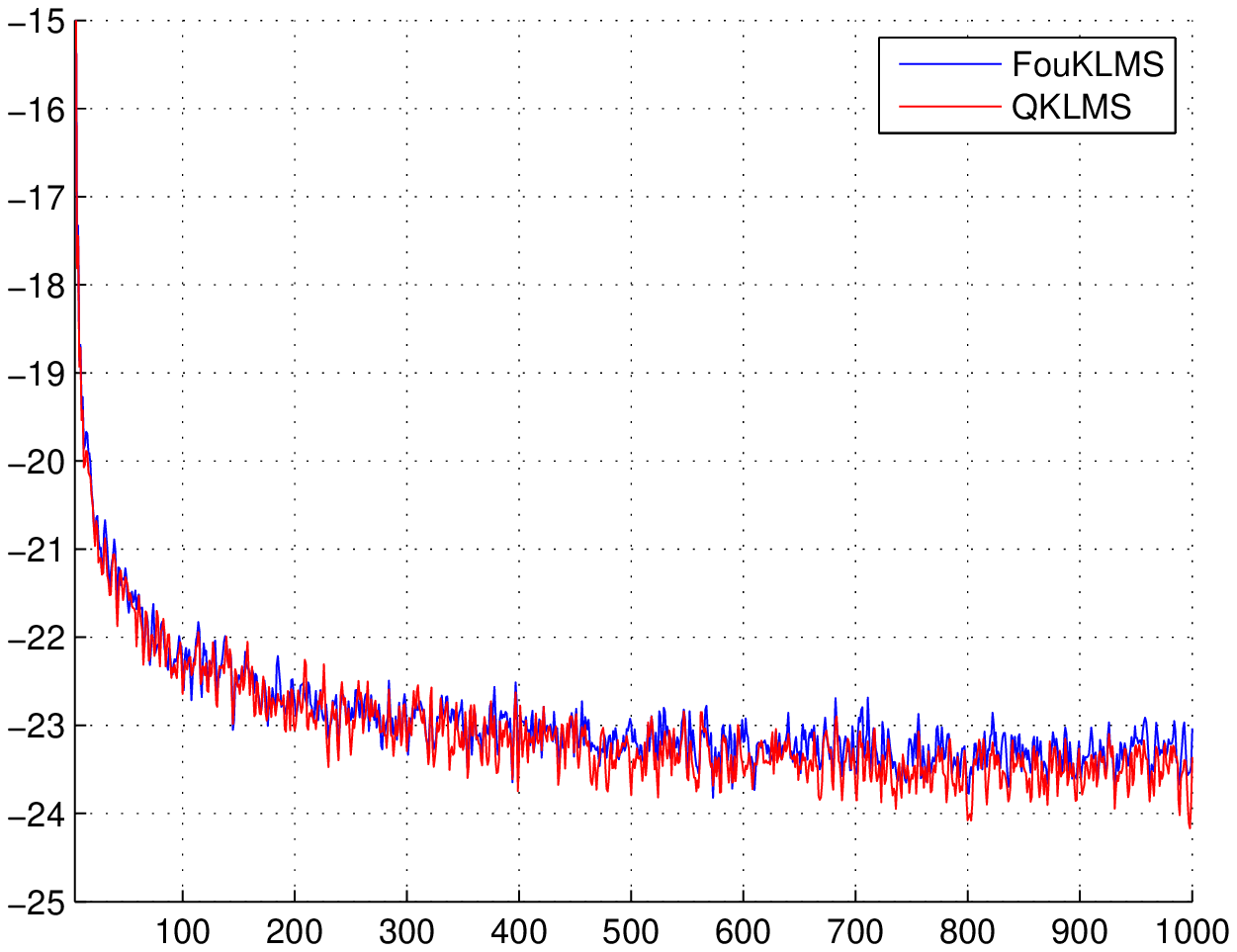,width=3.5cm}}
  \centerline{\scriptsize{(d) Example 8}}\medskip
\end{minipage}
\caption{Comparing the performances of RFFKLMS and the QKLMS.}
\label{FIG:RFFKLMS}
\end{figure}

\begin{table}[t]
\scriptsize
\caption{Mean training times for QKLMS and RFFKLMS.}\label{TAB:KLMS_times}
\centering
\begin{tabular}{|c|c|c|c|}
\hline
Experiment   &   QKLMS time  &  RFFKLMS time  &  QKLMS dictionary size\\\hline
Example 5    &   0.55 sec   &  0.35 sec       &  $M=1088$\\
Example 6    &   0.47 sec   &  0.15 sec       &  $M=104$\\
Example 7    &   0.02 sec   &  0.0057 sec     &  $M=7$\\
Example 8    &   0.03 sec   &  0.008 sec      &  $M=32$\\\hline
\end{tabular}
\vspace{-0.2cm}
\end{table}

\section{Conclusion}\label{SEC:CONCL}
We have presented a complete fixed-budget framework for non-linear online distributed learning in the context of RKHS. The proposed scheme achieves asymptotic consensus under some reasonable assumptions. Furthermore, we showed that the respective regret bound grows sublinearly with time. In the case of a network comprising only one node, the proposed method can be regarded as a fixed budget alternative for online kernel-based learning. The presented simulations validate the theoretical results and demonstrate the effectiveness of the proposed scheme.


%

\appendices
\section{Proof of Proposition \ref{PRO:regret}}\label{SEC:proof_regret}
In the following, we will use the notation $\cL_{k,n}(\btheta):=\cL(\bx_{k,n},y_{k,n},\btheta)$ to shorten the respective equations.
Choose any  $\bm{g}\in\mathcal{B}_{[\bm{0}_D,U_2]}$.
It holds that
\begin{align}
\| \bpsi_{k,n} &- \bm{g}\|^2 - \| {\btheta}_{k,n} - \bm{g}\|^2 =
-\| \bpsi_{k,n} - \btheta_{k,n} \|^2\nonumber\\
            &- 2\langle \btheta_{k,n} -\bpsi_{k,n},\bpsi_{k,n}- \bm{g}  \rangle
            = -\mu_{n}^2\|\nabla \cL_{k,n}(\bpsi_{k,n})\|^2\nonumber\\
            &+2\mu_n\langle \nabla \cL_{k,n}(\bpsi_{k,n}), \bpsi_{k,n} - \bm{g} \rangle.
\label{EQ:proof1}
\end{align}
Moreover, as $\cL_{k,n}$ is convex, we have:
\begin{align}
\cL_{k,n}(\btheta) \geq \cL_{k,n}(\btheta') + \langle \bm{h}, \btheta-\btheta'\rangle,
\label{EQ:proof2}
\end{align}
for all $\btheta,\btheta'\in\mathrm{dom}(\cL_{k,n})$ where $\bm{h}:=\nabla \cL_{k,n}(\btheta)$ is the gradient
(for a differentiable cost function) or a subgradient (for the case of a non--differentiable cost function).
From \eqref{EQ:proof1}, \eqref{EQ:proof2} and the boundness of the (sub)gradient we take
\begin{align}
\| \bpsi_{k,n} - \bm{g} \|^2 &- \|  {\btheta}_{k,n}- \bm{g} \|^2 \geq
-\mu_n^2 U^2 \nonumber\\ &-2\mu_n (\cL_{k,n}(\bm{g})-\cL_{k,n}(\bpsi_{k,n})),
 \label{EQ:proof3}
\end{align}
where $U$ is an upper bound for the (sub)gradient.
Recall that for the whole network we have: $\bm{\underline\psi}_n = \bA \bm{\underline\theta}_{n-1}$ and that for any doubly stochastic matrix, $\bA$, its norm equals to its largest eigenvalue, i.e., $\|\bA\|=\lambda_{\textrm{max}} = 1$. A respective eigenvector is $\bm{\underline{g}}=(\bm{g}^T\dots,\bm{g}^T)^T\in\R^{DK}$, hence it holds that $\bm{\underline{g}} = \bA\bm{\underline{g}}$ and
\begin{align}
\|\underline{\bpsi}_{n}-\underline{\bm{g}}\|
&=\|\bA\underline{\btheta}_{n-1}-\bA\underline{\bm{g}}\|
\leq \|\bA\Vert\Vert \underline{\btheta}_{n-1}-\underline{\bm{g}}\|\nonumber
\\  &=\|\underline{\btheta}_{n-1}-\underline{\bm{g}} \|
\label{EQ:proof4}
\end{align}
where $\underline{\bpsi}_n=(\bpsi_n^T,\dots,\bpsi_n^T)^T\in\R^{DK}$. Going back to \eqref{EQ:proof3} and summing over all $k\in\cN$,
   we have:
\begin{align}
\sum_{k\in\cN} (\| \bpsi_{k,n} &- \bm{g}\|^2 - \| \btheta_{k,n} - \bm{g}\|^2 )\geq \nonumber\\
 -\mu_n^2 K U^2 &-2\mu_n \sum_{k\in\cN}(\cL_{k,n}(\bm{g})-\cL_{k,n}(\bpsi_{k,n})).
 \label{EQ:proof5}
\end{align}
However, for the left hand side of the inequality
we obtain
$\sum_{k\in\cN} (\|  \bpsi_{k,n} - \bm{g} \|^2 - \| \btheta_{k,n} - \bm{g} \|^2 ) = \| \underline {\bpsi}_{n} - \underline{\bm{g}} \Vert^2 - \|\underline{\btheta}_{n} - \underline{\bm{g}}\|^2.$
If we combine the last relation with \eqref{EQ:proof4} and \eqref{EQ:proof5}
we have
\begin{align}
\| \underline{\btheta}_{n-1} - \underline{\bm{g}}\|^2 &- \|\underline{\btheta}_{n} - \underline{\bm{g}}\|^2 \geq \nonumber\\
 -\mu_n^2 K U^2 &-2\mu_n \sum_{k\in\cN}(\cL_{k,n}(\bm{g})-\cL_{k,n}({\bpsi}_{k,n})).
 \label{EQ:proof7}
\end{align}
The last inequality leads to
\begin{align*}
&\frac{1}{\mu_{n}}\| \underline{\btheta}_{n-1} - \underline{\bm{g}}\|^2 -\frac{1}{\mu_{n+1}} \|\underline{\btheta}_{n} - \underline{\bm{g}}\|^2 =\nonumber\\
&+\frac{1}{\mu_{n}}(\| \underline{\btheta}_{n-1} - \underline{\bm{g}}\|^2 - \| \underline{\btheta}_{n} - \underline{\bm{g}}\|^2)\nonumber\\
&+\left(\frac{1}{\mu_{n}}-\frac{1}{\mu_{n+1}}\right)\| \underline{\btheta}_{n} - \underline{\bm{g}}\|^2\geq \nonumber\\
& -\mu_n K U^2 -2 \sum_{k\in\cN}(\cL_{k,n}(\underline{\bm{g}})-\cL_{k,n}(\bpsi_{k,n}))\nonumber\\ &+4KU_2^2\left(\frac{1}{\mu_{n}}-\frac{1}{\mu_{n+1}}\right),
\end{align*}
where we have taken into consideration, Assumption 3 and the boundeness of $g$.
Next, summing over $i=1,\ldots,N+1$,  taking into consideration that  $\sum_{i=1}^{N} \mu_i\leq 2\mu \sqrt{N}$ (Assumption 1) and noticing that some terms telescope, we  have:
\begin{align*}
&\frac{1}{\mu}\| \underline{\btheta}_{0} - \underline{\bm{g}}\|^2 -\frac{1}{\mu_{N+1}} \Vert\underline{\btheta}_{N} - \underline{\bm{g}}\Vert^2
\geq -KU^2 2\mu\sqrt{N} \nonumber \\ &+ 2\sum_{i=1}^N\sum_{k\in\cN}(\cL_{k,i}(\bpsi_{k,i})-\cL_{k,i}(\bm{g}))+4KU_2^2\left(\frac{1}{\mu}-\frac{\sqrt{N+1}}{\mu}\right).
\end{align*}
Rearranging the terms and omitting the negative ones completes the proof:
\begin{align*}
\sum_{i=1}^N\sum_{k\in\cN}&(\cL_{k,i}(\bpsi_{k,i})-\cL_{k,i}(\bm{g}))  \\
&\leq \frac{1}{2\mu}\| \underline{\btheta}_{0} - \underline{\bm{g}}\|^2
 +KU^2 \mu\sqrt{N}  + 2KU_2^2 \frac{\sqrt{N+1}}{\mu}\\
 &\leq \frac{1}{2\mu}\| \underline{\btheta}_{0} - \underline{\bm{g}}\|^2
 +KU^2 \mu\sqrt{N}  + 2KU_2^2 \frac{\sqrt{N}+1}{\mu}.
\end{align*}


\section{Proof of Proposition \ref{PRO:DKLMS_cons}}\label{SEC:proof_DKLMS_cons}
For the whole network, the step update of RFF-DKLMS can be recasted as
\begin{align}
\underline{\btheta}_n &= \bA\underline{\btheta}_{n-1} + \mu \bV_n\underline{\bm{\varepsilon}}_n,\label{EQ:DKLMS_update}
\end{align}
where $\underline{\bm{\varepsilon}}_n = (\varepsilon_{1,n}, \varepsilon_{2,n}, \dots, \varepsilon_{K,n})^T$ and $\varepsilon_{k,n} = y_{k,n} - \bpsi_{k,n}^T\bz_\Omega(\bx_{k,n})$, or equivalently, $\underline{\bm{\varepsilon}}_n=\underline{\by}_n - \bV_n^T\bA\underline{\btheta}_{n-1}$. If we define $\bU_n = \underline{\btheta}_n - \underline{\btheta}_o$ and take into account that $\bA\underline{\bg} = \underline{\bg}$, for all $\underline{\bg}\in \R^{DK}$, such that $\underline{\bg} = (\bg^T, \bg^T, \dots, \bg^T)^T$ for $\bg\in\R^D$,  we obtain:
\begin{align*}
\bU_{n} &=  \bA\underline{\btheta}_{n-1} + \mu\bV_{n}(\underline{\by}_{n} - \bV_{n}^T\bA\underline{\btheta}_{n-1}) - \underline{\btheta}_o\\
&= \bA(\underline{\btheta}_{n-1} - \underline{\btheta}_{o}) + \mu\bV_{n}(\bV_{n}^T\underline{\btheta}_o + \underline{\bm{\epsilon}}_n + \underline{\bm{\eta}}_{n} - \bV_{n}^T\bA\underline{\btheta}_{n-1})\\
&= \bA\bU_{n-1} - \mu\bV_n\bV_n^T\bA\bU_{n-1} + \mu\bV_n\underline{\bm{\epsilon}}_n + \mu\bV_n\underline{\bm{\eta}}_n\\
\end{align*}
If we take the mean values and assume that $\btheta_{k,n}$ and $\bz_{\Omega}(\bx_{k,n})$ are independent for all $k=1,\dots, K$, $n=1,2,\dots,$ we have
\begin{align*}
E[\bU_{n}] = (I_{KD} - \mu\bm{R})\bA E[\bU_{n-1}] + \mu E[\bV_n \underline{\bm{\epsilon}}_n] + \mu E[\bV_n\underline{\bm{\eta}}_n].
\end{align*}
Taking into account that $\bm{\eta}_n$ and $\bV_n$ are independent, that $E[\bm{\eta}_n]=\bZero$ and that for large enough $D$ we have $E[\bV_n \bm{\epsilon}_n]\approx \bZero$, we can take
$E[\bU_n] \approx \left((I_{KD} - \mu\bm{R})\bA\right)^{n-1} E[\bU_1]$.
Hence, if all the eigenvalues of $(I_{KD} - \mu\bm{R})\bA$ have absolute value less than $1$, we have that $E[\bU_n]\rightarrow \bZero$. However, since $\bA$ is a doubly stochastic matrix we have $\|\bA\|\leq 1$ and
\begin{align*}
\|(I_{KD} - \mu\bm{R})\bA\| \leq \|I_{KD} - \mu\bm{R}\|\|\bA\| \leq \|I_{KD} - \mu\bm{R}\|.
\end{align*}
Moreover, as $I_{KD} - \mu\bm{R}$ is a diagonal block matrix, its eigenvalues are identical to the eigenvalues of its blocks, i.e., the eigenvalues of $I_{D} - \mu R_{zz}$. Hence, a sufficient condition for convergence is
$|1 - \mu\lambda_{D}(R_{zz})|<1$,
which gives the result.
\pend
\begin{remark}\label{REM:stab1}
Observe that $|\lambda_{\max}\left((I_{KD} - \mu\bm{R})\bA\right)|\leq |\lambda_{\max}\left((I_{KD} - \mu\bm{R})I_{KD}\right)|$, which means that the spectral radius of $(I_{KD} - \mu\bm{R})\bA$ is generally smaller than that of $(I_{KD} - \mu\bm{R})I_{KD}$ (which corresponds to the non-cooperative protocol). Hence, cooperation under the diffusion rationale has a stabilizing effect on the network \cite{Lopes}.
\end{remark}

\section{Proof of Proposition \ref{PRO:DKLMS_stab}}\label{SEC:proof_DKLMS_stab}
Let $\bm{B}_n = E[\bU_n\bU_n^T]$, where $\bU_n = \bA\bU_{n-1} - \mu\bV_n\bV_n^T\bA\bU_{n-1} + \mu\bV_n\underline{\bm{\epsilon}}_n + \mu\bV_n\underline{\bm{\eta}}_n$. Taking into account that the noise is i.i.d., independent from $\bU_n$ and $\bV_n$ and that $\bm{\epsilon}_n$ is close to zero (if $D$ is sufficiently large), we can take that:
\begin{align*}
\bm{B}_n =& \bA \bm{B}_{n-1}\bA^T - \mu \bA \bm{B}_{n-1}\bA^T \bm{R} - \mu \bm{R}\bA \bm{B}_{n-1}\bA^T\\
 & + \mu^2\sigma_\eta^2\bm{R} + \mu^2 E[\bV_n\bV_n^T\bA\bU_{n-1}\bU_{n-1}^T\bA^T\bV_n\bV_n^T].
\end{align*}
For sufficiently small step-sizes, the rightmost term can be neglected \cite{LMS_converg, Cattiveli10}, hence we can take the simplified form
\begin{align}\label{EQ:var_mat}
\bm{B}_n =& \bA \bm{B}_{n-1}\bA^T - \mu \bA \bm{B}_{n-1}\bA^T \bm{R} - \mu \bm{R}\bA \bm{B}_{n-1}\bA^T \nonumber\\
& + \mu^2\sigma_\eta^2\bm{R}.
\end{align}
Next, we observe that $\bm{B}_n$, $\bm{R}$ and $\bA$ can be regarded as block matrices, that consist of $K\times K$ blocks with size $D\times D$. We will vectorize equation \eqref{EQ:var_mat} using the $\vecbr$ operator, as this has been defined in \cite{Koning91}. Assuming a block-matrix $C$:
\begin{align*}
C = \tiny{\left(\begin{matrix}C_{11} & C_{12} & ... & C_{1K} \cr  C_{21} & C_{22} & ... & C_{2K} \cr \vdots & \vdots & & \vdots \cr C_{K1} & C_{K2} & ... & C_{KK}\end{matrix}\right)},
\end{align*}
the $\vecbr$ operator applies the following vectorization:
\begin{align*}
\vecbr C = ( \veco C_{11}^T, &\veco C_{12}^T, \dots, \veco C_{1K}^T,\dots, \\
& \veco C_{K1}^T \veco C_{K2}^T, \dots, \veco C_{KK}^T)^T.
\end{align*}
Moreover, it is closely related to the following block Kronecker product:
\begin{align*}
D \boxtimes C = {\tiny \left(\begin{matrix}D\otimes C_{11} & D\otimes C_{12} & ... & D\otimes C_{1K} \cr  D\otimes C_{21} & D\otimes C_{22} & ... & D\otimes C_{2K} \cr \vdots & \vdots & & \vdots \cr D\otimes C_{K1} & D\otimes C_{K2} & ... & D\otimes C_{KK}\end{matrix}\right)}.
\end{align*}
The interested reader can delve into the details of the $\vecbr$ operator and the unbalanced block Kronecker product in \cite{Koning91}. Here, we limit our interest to the following properties:
\begin{enumerate}
\item $\vecbr(D C E^T) = (E\boxtimes D) \vecbr C$.
\item $(C\boxtimes D)(E\boxtimes F) = CE \boxtimes DF$.
\end{enumerate}
Thus, applying the $\vecbr$ operator, on both sizes of \eqref{EQ:var_mat} we take
$\bm{b}_n = (\bA\boxtimes\bA)\bm{b}_{n-1} - \mu \left((\bm{RA})\boxtimes\bA\right)\bm{b}_{n-1} - \mu \left((\bm{A})\boxtimes\bm{RA}\right)\bm{b}_{n-1} + \mu^2\sigma_\eta^2\bm{r}$,
where $\bm{b}_n=\vecbr \bm{B}_n$ and $\bm{r} = \vecbr \bm{R}$. Exploiting the second property, we can take:
\begin{align*}
(\bm{RA})\boxtimes\bA = (\bm{RA})\boxtimes(\bm{I}_{DK}\bA) = (\bm{R}\boxtimes\bm{I}_{DK})(\bA\boxtimes\bA),\\
\bm{A}\boxtimes(\bm{RA}) = (\bm{I}_{DK}\bm{A})\boxtimes(\bm{RA}) = (\bm{I}_{DK}\boxtimes\bm{R})(\bA\boxtimes\bA).
\end{align*}
Hence, we finally get:
\begin{align*}
\bm{b}_n =& \left(\bI_{D^2K^2} - \mu\left(\bm{R}\boxtimes\bI_{DK}-\bI_{DK}\boxtimes\bm{A}\right)\right)(\bA\boxtimes\bA)\bm{b}_{n-1}\\
 &+ \mu^2\sigma_\eta^2\bm{r},
\end{align*}
which gives the result.


\ifCLASSOPTIONcaptionsoff
  \newpage
\fi



%
%
%
\bibliographystyle{IEEEtran}
\bibliography{athensbib}

%

\begin{IEEEbiography}[{\includegraphics[width=1in,height=1.25in,clip,keepaspectratio]{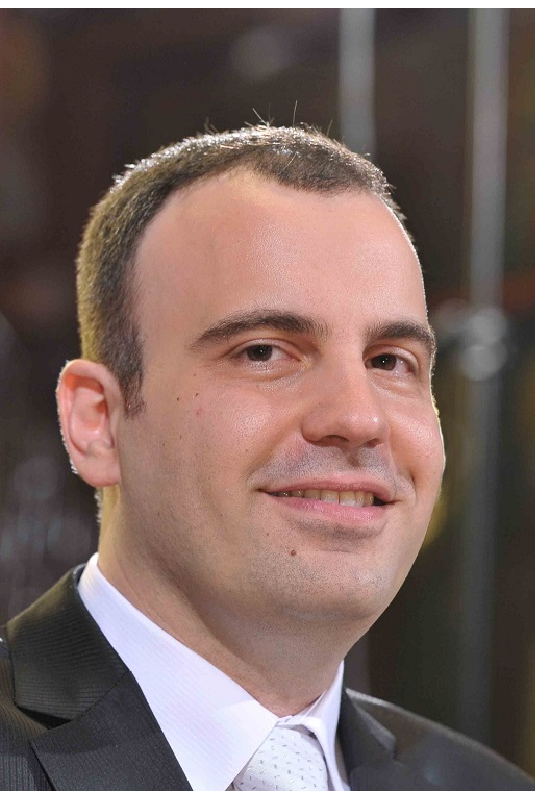}}]{Pantelis Bouboulis}
Pantelis Bouboulis received the B.Sc. degree in Mathematics and the M.Sc.
and Ph.D. degrees in Informatics and Telecommunications
from the National and Kapodistrian
University of Athens, Greece, in 1999, 2002 and 2006,
respectively. From 2007 till 2008, he served as an
Assistant Professor in the Department of Informatics
and Telecommunications, University of Athens. In 2010, he has received the Best scientific paper award for a work presented in the International Conference on Pattern Recognition, Istanbul, Turkey.
Currently, he is a Research Fellow at the Signal and
Image Processing laboratory of the department of Informatics
and Telecommunications of the University
of Athens and he teaches mathematics at the Zanneio
Model Experimental Lyceum of Pireas.  From 2012 since 2014, he served as an
Associate Editor of the IEEE Transactions of Neural Networks and Learning
Systems. His current research interests lie in the areas of machine learning,
fractals, signal and image processing.
\end{IEEEbiography}

\begin{IEEEbiography}[{\includegraphics[width=1in,height=1.25in,clip,keepaspectratio]{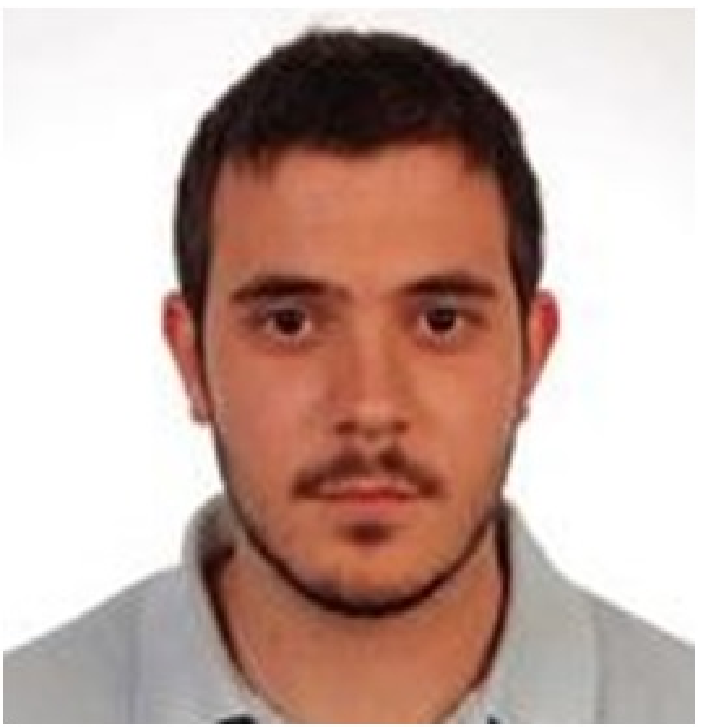}}]{Symeon Chouvardas}
Symeon Chouvardas received the B.Sc., M.Sc. (honors) and Ph.D. degrees from National and Kapodistrian University of  Athens, Greece, in 2008, 2011, and 2013, respectively. He was granted a Heracletus II Scholarship from GSRT (Greek Secretariat for Research and Technology) to pursue his PhD. In 2010 he was awarded with the Best Student Paper Award for the International Workshop on Cognitive Information Processing (CIP), Elba, Italy and in 2016 the Best Paper Award for the International Conference on Communications, ICC, Kuala Lumpur, Malaysia.  His research interests include: machine learning, signal processing, compressed sensing and online learning.
\end{IEEEbiography}


\begin{IEEEbiography}[{\includegraphics[width=1in,height=1.25in,clip,keepaspectratio]{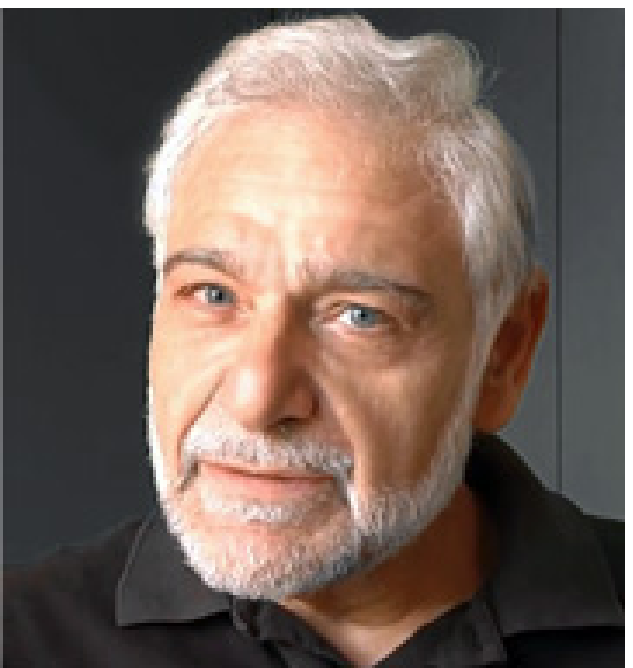}}]{Sergios Theodoridis}
(F' 08) is currently Professor of Signal Processing and Machine Learning in the Department of Informatics and Telecommunications of the University of Athens. His research interests lie in the areas of Adaptive Algorithms, Distributed and Sparsity-Aware Learning, Machine Learning and Pattern Recognition, Signal Processing for Audio Processing and Retrieval. He is the author of the book Machine Learning: A Bayesian and Optimization Perspective, Academic Press, 2015, the co-author of the best-selling book "Pattern Recognition, Academic Press, 4th ed. 2009, the co-author of the book "Introduction to Pattern Recognition: A MATLAB Approach, Academic Press, 2010, the co-editor of the bookœEfficient Algorithms for Signal Processing and System Identification, Prentice Hall 1993, and the co-author of three books in
Greek, two of them for the Greek Open University. He currently serves as Editor-in-Chief for the IEEE Transactions on Signal Processing. He is Editor-in-Chief for the Signal Processing Book Series, Academic Press and co-Editor in Chief
for the E-Reference Signal Processing, Elsevier. He is the co-author of seven papers that have received Best Paper Awards including the 2014 IEEE Signal Processing Magazine best paper award and the 2009 IEEE Computational Intelligence Society Transactions on Neural Networks Outstanding Paper Award. He is the recipient of the 2014 IEEE Signal Processing Society Education Award and the 2014 EURASIP Meritorious Service Award. He has served as a Distinguished Lecturer for the IEEE SP and CAS Societies. He was Otto Monstead Guest Professor, Technical University of Denmark, 2012, and holder of the Excellence Chair, Dept. of Signal Processing and Communications, University Carlos III, Madrid, Spain, 2011. He has served as President of the European Association for Signal Processing (EURASIP), as a member of the Board of Governors for the IEEE CAS Society, as a member of the Board of Governors (Member-at-Large) of the IEEE SP Society and as a Chair of the Signal Processing Theory and Methods (SPTM) technical committee of IEEE SPS. He is Fellow of IET, a Corresponding Fellow of the Royal Society of Edinburgh (RSE), a Fellow of EURASIP and a Fellow of IEEE.
\end{IEEEbiography}




\end{document}